\documentclass{article}
\PassOptionsToPackage{numbers, compress}{natbib}

 \usepackage[preprint]{neurips_2026}

\usepackage[utf8]{inputenc} % allow utf-8 input
\usepackage[T1]{fontenc}    % use 8-bit T1 fonts
\usepackage{hyperref}       % hyperlinks
\usepackage{url}            % simple URL typesetting
\usepackage{booktabs}       % professional-quality tables
\usepackage{amsfonts}       % blackboard math symbols
\usepackage{nicefrac}       % compact symbols for 1/2, etc.
\usepackage{microtype}      % microtypography
\usepackage[dvipsnames]{xcolor}       % colors
\usepackage{algorithm}
\usepackage{algorithmic}
\usepackage{amsmath}
\usepackage{amssymb}
\usepackage{mathtools}
\usepackage{amsthm}
\usepackage{bm}
\usepackage[capitalize,noabbrev]{cleveref}
\usepackage{makecell}
\usepackage{multirow}
\usepackage[textsize=tiny]{todonotes}
\usepackage{ascmac}
\usepackage{wrapfig}
\usepackage[most]{tcolorbox}
\usepackage{enumitem}

\theoremstyle{plain}
\newtheorem{theorem}{Theorem}
\newtheorem{proposition}{Proposition}

\theoremstyle{definition}

\newtheorem{assumption}{Assumption}
\theoremstyle{remark}

\title{Sharpness-Aware Hybrid Model Learning for Architecture-Agnostic Parameter Estimation}

\author{%
  Naoya Takeishi\\
  The University of Tokyo\\
  \texttt{ntake@g.ecc.u-tokyo.ac.jp}
}

\begin{document}

\maketitle

\begin{abstract}
Hybrid modeling, the combination of machine learning models and scientific mathematical models, enables flexible and robust data-driven prediction with partial interpretability. However, the unknown parameters of the scientific model cannot necessarily be estimated properly, since the flexibility of the machine learning model might make the scientific model part effectively ignored in prediction. We may avoid it by applying some regularization, but the formulation of such regularizers typically depends on model architectures and domain knowledge. In this paper, we propose an architecture-agnostic method to learn hybrid models while properly estimating the scientific parameters. The idea is to use the flatness of loss minima to achieve model simplicity, based upon the Occam's razor principle. We employ the idea of sharpness-aware minimization and adapt it to the hybrid modeling setting. Numerical experiments demonstrate the effectiveness of the SAM-based hybrid model learning for scientific parameter estimation.
\end{abstract}

\section{Introduction}\label{intro}

Combining scientific mathematical models and machine learning models can be useful for improving generalization, robustness, and partial interpretability of prediction.
The combination of these two contrasting regimes has been actively studied in various fields, including weather forecasting \citep{xuGeneralizingWeatherForecast2024,vermaClimODEClimateWeather2024}, robotics \citep{ajayAugmentingPhysicalSimulators2018,salzmannRealtimeNeuralMPC2023,heidenNeuralSimAugmentingDifferentiable2021,gaoSimtorealSoftRobots2024}, and healthcare \citep{millerLearningInsulinglucoseDynamics2020,kashtanovaDeepLearningModel2022,palumboHybridModelingPhotoplethysmography2025}.
However, such \emph{hybrid modeling} is, while being conceptually simple and intuitive, sometimes not technically straightforward.
When a machine learning model is too flexible compared to a scientific model to be combined with, hybrid modeling may become pointless, with the machine learning part alone fitting to all the data variation and the scientific model left unused \citep{yinAugmentingPhysicalModels2021,takeishiPhysicsintegratedVariationalAutoencoders2021,takeishiDeepGreyboxModeling2023,zouHybrid$^2$NeuralODE2024,wehenkelRobustHybridLearning2023}.
In other words, the unknown parameters of a scientific model, if any, may become unidentifiable regardless of their original identifiability.
Such a failure mode is indeed problematic because identifying the scientific parameters correctly, in some sense, is a key factor in ensuring the partial interpretation of hybrid models.

The identifiability issue in hybrid modeling can be mitigated by constraining the flexibility of the machine learning part so that the scientific model is adequately utilized.
For example, suppose an additive hybrid model $y = f(x) + g(x)$, where $f$ and $g$ are a scientific model and a neural net, respectively, and that both have unknown parameters.
In learning such a hybrid model, minimizing a loss along with $\Vert g \Vert^2$ sounds natural to suppress the excess flexibility of $g$, as practiced by \citet{yinAugmentingPhysicalModels2021}, because the two models are combined additively, and thus suppressing $g$'s energy would directly result in maximizing the use of $f$.
Now, how should we design a similar regularizer for a hybrid model with general composition, $y = g(x, f(x))$?
What about $y = g_3(x, f_1(x, g_2(x)) \cdot g_1(f_2(x)))$, and whatever more complicated models?
For hybrid architectures with various functional forms and nested function compositions, designing a good regularizer becomes far from trivial.

In this paper, we aim to develop a versatile method applicable to general hybrid models to partially address the unidentifiability of scientific parameters.
We hypothesize that \emph{seeking flat minima of a loss can help estimate scientific parameters}, at least to some extent.
Based on the hypothesis we propose to use a variant of sharpness-aware minimization (SAM) \citep{foretSharpnessawareMinimizationEfficiently2021} for learning deep hybrid models.
The proposed method does not depend on specific model architectures and thus should be useful broadly for different hybrid models.
We demonstrate the validity of the hypothesis through experiments with data from various physical systems and hybrid models with different architectures.

Our contribution lies in connecting two different lines of research that have never crossed.
On the one hand, loss flatness and its algorithmic implementations have been discussed in connection with generalization capability \citep[e.g.,][]{hintonKeepingNeuralNetworks1993,hochreiterFlatMinima1997,chaudhariEntropySGDBiasingGradient2017,dziugaiteComputingNonvacuousGeneralization2017,foretSharpnessawareMinimizationEfficiently2021}.
Generalization is not equivalent to parameter identification, as a ``correct'' scientific parameter does not necessarily lead to the best generalization in hybrid models.
On the other hand, regularization for hybrid models has been mostly developed in accordance with specific model architectures \citep{yinAugmentingPhysicalModels2021,takeishiPhysicsintegratedVariationalAutoencoders2021,takeishiDeepGreyboxModeling2023,zouHybrid$^2$NeuralODE2024,wehenkelRobustHybridLearning2023}, and the property of a loss surface has not been considered relevant.
In this paper, we shed light on the utility of loss flatness for hybrid model learning, opening up a new pathway toward architecture-agnostic parameter estimation in hybrid models. 

\section{Background: Hybrid Modeling}

Suppose a prediction task to predict $y \in \mathcal{Y} \subset \mathbb{R}^{d_y}$ from $x \in \mathcal{X} \subset \mathbb{R}^{d_x}$.
This restriction is only for notational simplicity, and the discussion below can be easily extended to more general settings.
Let $f_\theta: \mathcal{X} \to \mathcal{Z}$ and $g_\phi: \mathcal{X} \times \mathcal{Z} \to \mathcal{Y}$ be a scientific mathematical model and a machine learning model (e.g., a deep neural net) parametrized by $\theta \in \Theta \subset \mathbb{R}^{d_\theta}$ and $\phi \in \Phi \subset \mathbb{R}^{d_\phi}$, respectively.
We denote the output space of the scientific model by $\mathcal{Z}$.
Importantly, we suppose that the scientific model, $f_\theta$, is incomplete in some sense to describe data precisely, which makes it meaningful to combine the scientific model with the data-driven, machine learning model.

Unless stated otherwise, we consider the combination of the two as a general function composition:
\begin{equation}\label{eq:general_hybrid_model}
    h_{\theta,\phi}(x) = g_\phi(x, f_\theta(x)).
\end{equation}
While more complex, nested compositions can be considered as well, to avoid clutter we stick to this simple form, which already may cause the issue discussed below.
Note also that hybrid models in practice may contain further operations such as numerical integration over time; for example, it is typical to use hybrid neural ODEs, $y = \operatorname{ODESolve} ( h_{\theta,\phi}(x;t)=0 )$, in modeling dynamical systems \citep{yinAugmentingPhysicalModels2021}.
We omit such operations for notational simplicity.

We informally assume the following:
First, $g_\phi$ is sufficiently flexible to approximate arbitrary functions from (compact subsets of) $\mathcal{X} \times \mathcal{Z}$ to $\mathcal{Y}$; this is usually the case when $g_\phi$ is a deep neural net.
Second, we have data of $x \in \mathcal{X}$ and $y \in \mathcal{Y}$ only, and the values of $z \in \mathcal{Z}$ are not observable as data; if they were, learning hybrid models would become trivial.
Third, $f_\theta$'s parameters, $\theta$, would be identifiable \emph{if} we had sufficient data of $x$ and $z$; in other words, the scientific model itself originally possessed structural and practical identifiability \citep[see, e.g.,][]{guillaumeIntroductoryOverviewIdentifiability2019,anstett-collinPrioriIdentifiabilityOverview2020,wielandStructuralPracticalIdentifiability2021}.

Due to the flexibility of $g_\phi$ and the unobservability of $z$, we face a risk that the hybrid model, $h_{\theta,\phi}$, fits all the data variation by adjusting only the neural net's parameters $\phi$, with arbitrary values of the scientific model's parameters $\theta$.
Let $L(\theta, \phi)$ denote the (population) loss function.
In terms of estimating $\theta$, merely minimizing the profile loss,
\begin{equation*}
    \min_\theta J(\theta) \quad\text{where}\quad J(\theta) \coloneqq \min_\phi L(\theta, \phi),
\end{equation*}
does not identify $\theta$ when the profile loss uniformly reaches a small level $\varepsilon$, i.e., $ J(\theta) \approx \varepsilon$ for all $\theta$, and thus the outer problem does not have clear minima (or has too many insignificant ones) in the space of $\theta$.
In the meantime the model can always fit the data due to the flexibility of $g_\phi$.
Consequently, the scientific model, $f_\theta$, would be left unused, and the merits of hybrid modeling would be lost \citep{yinAugmentingPhysicalModels2021,takeishiPhysicsintegratedVariationalAutoencoders2021,takeishiDeepGreyboxModeling2023}.

A natural circumvention to break the tie in the profile loss is to incorporate a constraint or regularizer.
Such a regularizer should suppress the excess flexibility of the machine learning part and encourage the use of the scientific model in prediction.
For instance, consider an additive hybrid model $h_{\theta,\phi}(x) = f_\theta(x) + g_\phi(x)$.
\citet{yinAugmentingPhysicalModels2021} suggested suppressing the flexibility of $g_\phi$ by penalizing its functional norm, $\Vert g_\phi \Vert^2$.
Valid formulation of such a regularizer heavily depends on task definition and model architectures; unlike the additive form that allows an intuitive design of the regularizer, it is unclear how we should regularize more complicated hybrid architectures.

\section{Learning Hybrid Models with Simplicity Principle}
\label{method}

\subsection{Idea}

A typical strategy in learning hybrid models is to add a regularizer, i.e., to minimize
\begin{equation}\label{eq:regularized_profile_loss}
    J_R(\theta) \coloneqq \min_\phi \Big( L(\theta, \phi) + R(\theta, \phi) \Big),
\end{equation}
where $R(\theta, \phi)$ is some regularization term for suppressing the excess flexibility of $g_\phi$.
For example, \citet{yinAugmentingPhysicalModels2021} suggested using $R \propto \| g_\phi \|^2$ for additive hybrid models.
\citet{takeishiPhysicsintegratedVariationalAutoencoders2021} proposed to define $R$ as the discrepancy between the full hybrid model and a reduced version where $g_\phi$ was somehow disabled, e.g., by replacing it with a simpler function.
Standard regularizers, such as parameter norm $R \propto \|\phi\|^2$, are known to be not effective particularly for the estimation of $\theta$ \citep{yinAugmentingPhysicalModels2021,takeishiPhysicsintegratedVariationalAutoencoders2021}.
Good design of $R$ heavily depends on the architecture of hybrid models, which hinders efficient development of various hybrid model architectures with complicated function compositions.

Toward an architecture-agnostic learning strategy, we step back to the Occam's razor principle that simpler models are preferred when multiple models explain the data equally well.
We apply it to hybrid model learning, assuming the following intuitive property:
\begin{tcolorbox}[colframe=SkyBlue,colback=SkyBlue!10!White,top=2pt,bottom=2pt,enlarge top by=-3pt,enlarge bottom by=-3pt]
    \centering
    good scientific parameter $\theta$ \ $\Leftrightarrow$ \ the residual task left to $g_\phi$ is simple and short to describe.
\end{tcolorbox}
We show the condition for identification more formally in Appendix~\ref{appendix:identifiability}.
The idea can be implemented by using the free energy:
\begin{equation}\label{eq:soft_profile_loss}
    \bar{J}(\theta) \coloneqq \min_q \Big( \mathbb{E}_{\phi \sim q} L(\theta, \phi) + \operatorname{KL}(q(\phi) \ \| \ p(\phi)) \Big),
\end{equation}
where $p$ and $p$ are prior and posterior distributions of $\phi$, respectively.
In \eqref{eq:soft_profile_loss}, we consider distributions on $\phi$, unlike the point estimation in \eqref{eq:regularized_profile_loss}.
The KL term is necessary to prevent point-mass $q$.
The minimizer of the $\min_q$ in \eqref{eq:soft_profile_loss} is known to be $q=q^*_\theta (\phi)\propto p(\phi) \exp(-L(\theta,\phi))$ \citep[see, e.g.,][]{dupuis_ellis_1997}.
% Proposition 1.4.2

Why does using \eqref{eq:soft_profile_loss} as a loss realize simplicity-based learning?
We can understand it based on the well-known argument about its relation to description length \citep[e.g.,][]{hintonKeepingNeuralNetworks1993,hintonAutoencodersMinimumDescription1993,honkelaVariationalLearningBitsback2004,gravesPracticalVariationalInference2011a}.
Suppose $L$ is the negative log likelihood, i.e., $L(\theta,\phi) = -\log \pi(\mathcal{D} \mid \theta,\phi)$, where $\mathcal{D}$ and $\pi$ denote data and a likelihood function.
Substituting the minimizer $q^*_\theta$ back to \eqref{eq:soft_profile_loss}, we have
\begin{equation}\label{eq:description_length}
    \bar{J}(\theta) = - \mathbb{E}_{\phi \sim q^*_\theta} \log \pi(\mathcal{D} \mid \theta,\phi) + \operatorname{KL}(q^*_\theta(\phi) \ \| \ p(\phi)).
\end{equation}
The first term of \eqref{eq:description_length} is the expected cost of coding the data given $\phi$ (and $\theta$), and from the ``bits back'' argument \citep{hintonKeepingNeuralNetworks1993}, the second term is the expected extra cost of coding $\phi$
itself using $p$ as the prior code.
Minimizing $\bar{J}$ therefore achieves a simple model in the sense of the minimum description length.

It has also been known that the loss in \eqref{eq:soft_profile_loss} or \eqref{eq:description_length} is related to the flatness of loss minima \citep[e.g.,][]{hintonKeepingNeuralNetworks1993,hochreiterFlatMinima1997,dziugaiteComputingNonvacuousGeneralization2017,chaudhariEntropySGDBiasingGradient2017,mollenhoffSAMOptimalRelaxation2022}.
Suppose that a wrong estimation of $\theta$ can only be made good (in terms of small $L$) by a very specially tuned value of $\phi$, namely $\phi_\theta$, corresponding to a non-flat, sharp minimum.
The original profile loss, $J$, may exploit such $\phi_\theta$, as the argument of $\min_\phi$ can indeed point to $\phi_\theta$.
On the other hand, $\bar{J}$ penalizes such $\phi_\theta$ because for $\phi \neq \phi_\theta$ nearby, the value of $L(\theta,\phi)$ is not small, and thus $\mathbb{E}_{\phi \sim q} L(\theta, \phi)$ becomes large.
We thus use an optimization method that takes the flatness into account, as presented in the next section.

The hypothetical benefit of $\bar{J}$ is based on the Occam's razor principle and cannot be automatically guaranteed for all problem settings, without additional assumptions about the data distribution and the model architecture.
We thus examine the empirical performance of the proposed method in \cref{experiment}.

\subsection{SAM for Hybrid Models}
\label{method:main}

Although we can solve the inner problem of $\min_\theta \min_q (\cdots)$ exactly like \eqref{eq:description_length}, solving the outer problem directly is challenging because the computation and optimization of the normalizing constant of $q^*_\theta$, which depends on $\theta$, is usually not tractable.
Thus in practice, we think of minimizing
\begin{equation}\label{eq:minimand}
    \mathbb{E}_{\phi \sim q} L(\theta, \phi) + \operatorname{KL}(q(\phi) \ \| \ p(\phi)),
\end{equation}
simultaneously for $\theta$ and (sufficient statistics of) $q(\phi)$.
However, even if we restrict $q$ to a tractable family, taking the expectation about high-dimensional $\phi$ is inefficient and unstable.

We thus suggest relaxing the problem and using SAM \citep{foretSharpnessawareMinimizationEfficiently2021} to seek flat (non-sharp) minima of the loss in learning hybrid models.
See Appendix~\ref{appendix:sam} for a general explanation of SAM and its variants.
The overall procedures are shown in \cref{alg:main}.
At each step we first compute the direction $\epsilon^*_\phi$ to ascend the loss surface.
We then compute the gradient at the perturbed point, go back to the original point storing the gradient, and finally do the update with the gradient.

\begin{wrapfigure}[8]{r}[0pt]{.45\textwidth}
    \centering
    \vspace*{-2ex}
    \includegraphics[clip,width=0.95\linewidth]{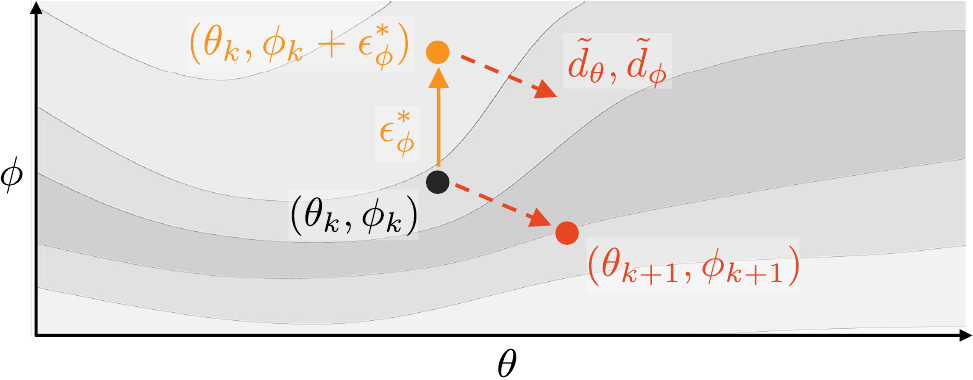}
    \vspace*{-8pt}
    \caption{Update step of \cref{alg:main}.}
    \label{fig:hybrid_sam}
\end{wrapfigure}
\cref{alg:main} is slightly modified from ordinary SAMs \citep{foretSharpnessawareMinimizationEfficiently2021} in that the parameter is perturbed only in the $\phi$-direction, as depicted in \cref{fig:hybrid_sam}.
Because it is the machine learning part $g_\phi$ that should be simple for the maximum use of the scientific part $f_\theta$, SAM should reach a value of $\phi$ around which the loss is flat along the $\phi$-axis.
On the other hand, nothing requires the flatness along the $\theta$-axis; on the contrary, being flat along $\theta$ means that different values of $\theta$ achieve similarly small loss values, implying $\theta$ is hard to identify.
% The inherent identifiability of $\theta$ depends on the structural and practical identifiability for the scientific model $f_\theta$ and should not be harmed.
In practice, however, the effect of the modification is not necessarily dominant in performance because a scientific model usually has a small number of parameters.

\begin{algorithm}[t]
    \caption{SAM for hybrid models}
    \label{alg:main}
    \begin{algorithmic}[1]
        \REQUIRE Number of epochs $M$, data $S$, loss $\ell(\cdot,\cdot)$, learning rates $\eta_\theta,\eta_\phi$, perturbation radius $\rho_\phi$
        \ENSURE Learned parameters $\theta,\phi$
        \STATE Initialize $\theta,\phi$
        \FOR{epoch $=1$ to $M$}
            \FOR{each minibatch $B \subset S$}
                \STATE $L_B(\theta,\phi) \gets \frac{1}{|B|}\sum_{(x,y)\in B}\ell\!\left(h_{\theta,\phi}(x),y\right)$
                \STATE $\epsilon^*_\phi \gets$ ascend direction by \eqref{eq:eps_sam}, \eqref{eq:eps_asam}, or \eqref{eq:eps_fsam} (or other variants)
                \STATE $\tilde d_\theta \gets \nabla_\theta L_B(\theta,\phi+\epsilon^*_\phi)$, \quad $\tilde d_\phi \gets \nabla_\phi L_B(\theta,\phi+\epsilon^*_\phi)$
                \STATE $\theta \gets \theta - \eta\, \tilde d_\theta$, \quad $\phi \gets \phi - \eta\, \tilde d_\phi$
            \ENDFOR
        \ENDFOR
    \end{algorithmic}
\end{algorithm}

\paragraph{Analysis}
Let us see how SAM is relevant to the original minimization objective in \eqref{eq:minimand}.
Suppose the prior $p(\phi)$ and posterior $q(\phi)$ are Gaussian.
Then the KL term in \eqref{eq:minimand} can be computed analytically.
We thus want to approximate the remaining, expected population loss $\mathbb{E}_{\phi \sim q} L(\theta, \phi)$.
We can show that our SAM variant in \cref{alg:main} minimizes an upper bound of it.
Let $S \coloneqq \{(x_i,y_i) \mid i=1,\dots,n\}$ be a training dataset with i.i.d. samples.
The objective of SAM for hybrid models in \cref{alg:main} is
\begin{equation*}
    L_S^\text{sam}(\theta, \phi) \coloneqq \max_{\|\epsilon\| \leq \rho} L_S(\theta, \phi + \epsilon),
    \quad \text{where} \quad
    L_S(\theta, \phi) \coloneqq \frac1n \sum_{(x,y) \in S} \ell(h_{\theta,\phi}(x), y)).
\end{equation*}
The expected population loss is bounded by the SAM objective as follows:
\begin{theorem}[informal]
    Suppose that $\ell$ is bounded and Lipschitz in $\theta$, and that $\Theta$ is a bounded set. Let $q \!=\! q_\varphi \coloneqq \mathcal{N}(\phi \mid \varphi, \tau^2 I)$ for $\varphi \in \Phi$ and $\tau^2 > 0$.
    Then, with high probability over the choice of $S$,
    \begin{equation*}
        \mathbb{E}_{\phi \sim q_\varphi} L(\theta, \phi)
        \leq L_S^\text{sam}(\theta, \varphi)
        + \mathcal{C}_n (\varphi, \Theta)
        + o(1),
    \end{equation*}
    for all $\theta \in \Theta$ and $\varphi \in \Phi$, where $\mathcal{C}_n (\varphi, \Theta)$ is a complexity term that is increasing in both $\|\varphi\|^2$ and the metric entropy of $\Theta$, and decreases as $\widetilde{O}(1/\sqrt{n})$. The $o(1)$ term comes from a Gaussian tail and decays exponentially in $d_\phi$.
\end{theorem}
\begin{proof}
    See Appendix~\ref{appendix:theory}.
    The proof is based on the PAC-Bayes argument of \citet{foretSharpnessawareMinimizationEfficiently2021} with a uniform-convergence argument based on covering numbers of $\Theta$.
\end{proof}

Moreover, \citet{mollenhoffSAMOptimalRelaxation2022} analyzed the relation between the variational formulation of Bayesian inference, like \eqref{eq:minimand}, and SAM.
They show the Bayes objective can be lower bounded by a soft version of SAM, where the perturbation $\epsilon$ is chosen with a soft penalty term $\propto \|\epsilon\|^2$, instead of the hard constraint $\|\epsilon\| \leq \rho$.
From this fact they point out (a variant of) SAM solves a relaxation of the Bayes problem, which further motivates us to use SAM to tackle the minimization of \eqref{eq:minimand}.

\section{Related Work}

\subsection{Hybrid Modeling}

Hybrid models (also called grey-box models, residual physics, etc.) combining scientific mathematical models (also called mechanistic models, physics models, expert models, etc.) and machine learning models have been studied for decades \citep[e.g.,][]{psichogiosHybridNeuralNetworkfirst1992,rico-martinezContinuoustimeNonlinearSignal1994,thompsonModelingChemicalProcesses1994,forssellCombiningSemiphysicalNeural1997,schweidtmannReviewPerspectiveHybrid2024}.
Recently researchers have been interested in using deep neural nets in hybrid modeling \citep[e.g.][]{yinAugmentingPhysicalModels2021,takeishiPhysicsintegratedVariationalAutoencoders2021,takeishiDeepGreyboxModeling2023,qianIntegratingExpertODEs2021,haussmannLearningPartiallyKnown2021,mehtaNeuralDynamicalSystems2021,wehenkelRobustHybridLearning2023,tusharDeepPhysicsCorrector2023,rudolphHybridModelingDesign2024,holtAutomaticallyLearningHybrid2024,yeIdentifiabilityHybridDeep2024,giampiccoloRobustParameterEstimation2024,cohrsCausalHybridModeling2024,claesHybridAdditiveModeling2025,thoreauPhysicsinformedVariationalAutoencoders2025,singhVariationalGreyboxDynamics2026}.
The application of deep hybrid modeling has been actively practiced in various fields such as traffic prediction \citep{shirakamiQTNetTheorybasedQueue2023}, weather forecasting \citep{xuGeneralizingWeatherForecast2024,vermaClimODEClimateWeather2024}, robotics \citep{ajayAugmentingPhysicalSimulators2018,salzmannRealtimeNeuralMPC2023,heidenNeuralSimAugmentingDifferentiable2021,gaoSimtorealSoftRobots2024}, disease prediction \citep{arikInterpretableSequenceLearning2020}, and healthcare \citep{millerLearningInsulinglucoseDynamics2020,kashtanovaDeepLearningModel2022,palumboHybridModelingPhotoplethysmography2025}.

Several researchers have been working on methods to learn deep hybrid models appropriately to ensure the maximum use of scientific mathematical models.
\citet{yinAugmentingPhysicalModels2021} discuss hybrid neural ODEs and propose a method to learn the model by minimizing the functional norm of the neural network part.
It is a natural formulation for the additive combination of the scientific model and the neural net, but is not straightforwardly applicable to other forms of hybrid models.
\citet{takeishiPhysicsintegratedVariationalAutoencoders2021} propose a regularizer that minimizes the discrepancy between the full hybrid model and a scientific-model-only reduction and apply it to variational autoencoders whose decoder has a hybrid structure.
It is in principle applicable to general architectures as long as the reduced model can be defined, but the reduced model needs to be designed for each architecture.
We will examine the performance of these methods in the experiments in \cref{experiment}.

Some methods utilize more specific model architectures and additional information.
\citet{zouHybrid$^2$NeuralODE2024} present a method to use domain knowledge of the effects of interventions to maintain meaningful causal implications.
\citet{claesHybridAdditiveModeling2025} propose a principled method to properly learn hybrid models when the two parts work on distinct feature sets.
\citet{cohrsCausalHybridModeling2024} suggest using the method of double machine learning for specific types of hybrid models.
In the context of simulation-based inference with misspecified simulators, which is very close to the idea of hybrid modeling, \citet{wehenkelAddressingMisspecificationSimulationbased2025} and \citet{senoufInductiveDomainTransfer2025} utilize a small amount of calibration data, i.e., data of the output of the scientific model.
These methods would certainly work well when the assumptions are met, but in this paper, we do not assume such specific architectures or the presence of calibration data.

\subsection{Flatness and Generalization}

Flatness of loss minima has been discussed in connection with generalization capability \citep[e.g.,][]{hintonKeepingNeuralNetworks1993,hochreiterFlatMinima1997,baldassiUnreasonableEffectivenessLearning2016,dziugaiteComputingNonvacuousGeneralization2017,chaudhariEntropySGDBiasingGradient2017,keskarLargebatchTrainingDeep2017,izmailovAveragingWeightsLeads2018}.
Flat minima are intuitively understood as wide valleys in the loss landscape, where the loss remains low even if the parameters are perturbed.
Researchers have studied the advantage of flat minima in generalization, both theoretically and empirically, often with concrete evidence of performance improvement in well-known tasks such as image classification.

One of the rationales of flat minima for better generalization is the correspondence between flatness and model simplicity.
For example, \citet{hintonAutoencodersMinimumDescription1993} and \citet{hochreiterFlatMinima1997} connected flat minima with model complexity in terms of minimum description length \citep{rissanenModelingShortestData1978}.
At a flat minimum, relatively low precision is needed to describe the model's parameters maintaining the performance, thus the model is simple in the sense that the description length is short.
\citet{hochreiterFlatMinima1997} also argued that mere flatness is not sufficient to ensure a short description length and that a minimum should be equally flat for each parameter.
It is also notable that, as discussed by \citet{dinhSharpMinimaCan2017}, the flatness of a loss function may be arbitrarily scaled without altering a function's outcomes.
Such discussions imply the importance of careful definition of flatness \citep[see, e.g.,][]{tsuzukuNormalizedFlatMinima2020}.

\section{Experiment}
\label{experiment}

\subsection{Scope}

The aim of the experiment is to examine the effectiveness of SAM-based learning of hybrid models.
Importantly, we are interested in \emph{how the method helps the estimation of scientific parameters}, especially for architectures to which the existing regularization methods \citep{yinAugmentingPhysicalModels2021,takeishiPhysicsintegratedVariationalAutoencoders2021} are not straightforward to apply.
Among the tasks introduced in \cref{experiment:tasks}, the first two tasks have already been studied in previous studies \citep[e.g.][]{yinAugmentingPhysicalModels2021}, and thus the existing methods are applicable.
Applying the existing methods to the other tasks is not straightforward or requires careful design of the regularizer.

Obviously, the performance of $\theta$ estimation can be measured only when some true or reference values of $\theta$ are available.
We prepared the synthetic data and referred to the domain knowledge for the real-world data.
% Any experiments with quantitative evaluation of the estimation must be controlled in this regard.
% The point of the empirical evaluation here is to see if the SAM-based method, which does not assume specific model architectures, can work comparably with the learning methods whose exact formulation depends on architectures.
Several methods of hybrid model learning are out of the scope of the current experiments or are not applicable to our tasks.
The hybrid models to be tested in the experiments have unknown parameters in both of the scientific and machine learning parts; thus the learning methods used in some hybrid modeling studies \citep[e.g.,][]{haussmannLearningPartiallyKnown2021}, where the scientific model parameters are fixed, are not of our main interest here.
We do not assume at all the access to calibration data, i.e., data of the output of the scientific model, $z = f_\theta(x)$; thus the methods based on partial access to calibration data \citep{wehenkelAddressingMisspecificationSimulationbased2025,senoufInductiveDomainTransfer2025} are out of the scope.
Moreover, none of the methods assuming specific hybrid model architectures or additional information \citep{zouHybrid$^2$NeuralODE2024,cohrsCausalHybridModeling2024,claesHybridAdditiveModeling2025} are applicable to our tasks.

\subsection{Tasks}
\label{experiment:tasks}

The full details of the tasks can be found in Appendix~\ref{appendix:tasks}.

\subsubsection{Synthetic Data Tasks}
We prepared four synthetic data from different physical systems for fully controlled experiments.

\paragraph{\textsc{Pendulum time-series}}
As practiced in previous studies \citep[e.g.,][]{yinAugmentingPhysicalModels2021,takeishiPhysicsintegratedVariationalAutoencoders2021,wehenkelRobustHybridLearning2023}, we generated data from the damped pendulum system $\ddot{v}(t) + \gamma \dot{v}(t) + (2 \pi \omega)^2 \sin v(t)=0$, where $v(t)$ is a pendulum's angle, and $\ddot{v}$ and $\dot{v}$ are the second and first derivatives with regard to $t$, respectively.
After generating data, we assume no access to the data-generating values of $\gamma$ and $\omega$.
We formulate the task as a sequence prediction from $x = \mathbf{v}(0) \in \mathbb{R}^{2}$ to $y = [ \mathbf{v}(\Delta t), \dots, \mathbf{v}(m_y \Delta t) ] \in \mathbb{R}^{m_y \times 2}$, where $\mathbf{v}=[v, \dot{v}]$, $\Delta t$ is the time step, and $m_y$ is the output sequence length.

To formulate a hybrid model, we suppose we only know a part of the equation as an incomplete scientific model $f_\theta(v) := \ddot{v}(t) + (2 \pi \tilde\omega)^2 \sin v(t)$, where $\theta = \{ \tilde\omega \}$ is the unknown parameter to be identified.
Then a hybrid neural ODE is formulated as
\begin{equation}\label{eq:hybrid_node}
    y = \operatorname{ODESolve} ( f_\theta(v) + g_\phi(v) = 0; v(0) = x ),
\end{equation} where $g_\phi$ is a feed-forward neural net with two hidden layers and the ReLU activation function.

\paragraph{\textsc{Reaction-diffusion}}
We generated data from a reaction-diffusion system: $u_t(t,\xi) = a \nabla^2 u + u - u^3 - \kappa - v$, $v_t(t,\xi) = b \nabla^2 v + u - v$, where $u(t,\xi)$ and $v(t,\xi)$ are the concentrations of two chemical species at time $t$ and spatial location $\xi \in [-1,1]^2$, $u_t$ and $v_t$ are the temporal derivatives of $u$ and $v$, respectively, and $\nabla^2$ is the Laplacian operator with regard to the space $\xi$.
The task is again to predict a sequence of the states $y \in \mathbb{R}^{m_y \times 2 \times d \times d}$ given an initial condition $x \in \mathbb{R}^{2 \times d \times d}$, where $d=32$ is the number of spatial discretization points along each axis.
We build a hybrid neural PDE model with an incomplete scientific model as the diffusion terms: $f_\theta(u,v) := [ u_t - \tilde a \nabla^2 u, \ \ v_t - \tilde b \nabla^2 v]^\top$, where $\theta = \{ \tilde a, \tilde b \}$ is the unknown parameter to be identified.
The machine learning part $g_\phi$ is a convolutional neural net.
The two models are combined like in \eqref{eq:hybrid_node}.

\paragraph{\textsc{Duffing oscillator}}
We generated data from a nonlinear oscillator system: $\ddot{v}(t) + \gamma \dot{v}(t) + \alpha v(t) - \beta v(t)^3 = 0$.
We formulate the task similarly to the preceding tasks, i.e., the prediction from an initial condition $x \in \mathbb{R}^{2}$ to a subsequent states $y \in \mathbb{R}^{m_y \times 2}$.
The hybrid model to be used is again a hybrid neural ODE like \eqref{eq:hybrid_node} with an incomplete scientific model: $f_\theta(v) \coloneqq \ddot{v}(t) + \tilde\alpha v(t)$, where $\theta = \{ \tilde\alpha \}$ is the unknown parameter to be estimated.
We used a feed-forward neural net as $g_\phi$.

\paragraph{\textsc{Pendulum images}}
We first generated trajectories of a pendulum as in the \textsc{pendulum time-series} task and then created gray-scale images of size $d \times d$ with $d=48$ pixels based on the simulated pendulum angles.
The task is to predict future image sequences $y \in [0,1]^{m_y \times d \times d}$ given an initial image $x \in [0,1]^{d \times d}$.
The hybrid model is built by composing the hybrid neural ODE in \eqref{eq:hybrid_node} and convolutional neural nets for decoding and encoding.
The decoder maps the pendulum angle to an image, and the encoder maps the input image to a latent representation including the initial condition of the pendulum.
Thus the overall architecture of the hybrid model is more complicated than the mere neural ODE with the additive combination.
Note that the pendulum's initial condition is inherently unidentifiable; we are only interested in estimating the pendulum's frequency, i.e., $\theta=\{ \tilde\omega \}$.

\subsubsection{Real-World Data Tasks}

We used two real-world datasets that have realistic scales as scientific applications.

\paragraph{\textsc{Wind tunnel}}

We used a dataset comprising measurements from a real-world wind tunnel \citep{gamellaCausalChambersRealworld2025}.
We extracted a subset of the data to formulate a task to predict $y \in \mathbb{R}^{m_y}$ (time-seires of the pressure inside the tunnel) from $x \in \mathbb{R}^{m_x \times 3}$ (time-series of the loads of the intake and exhaust fans and the position of the hatch).
The hatch controls an additional opening in the middle of the tunnel.
As a ``scientific'' model, we use the models ``A1'' and ``C2'' presented in \citet{gamellaCausalChambersRealworld2025}, which give a rough relation between the pressure and the fan loads.
These models are incomplete in terms of the effect of the hatch position and the transient dynamics.
The hybrid model transforms the output of these models with a feed-forward neural net, thus the overall architecture is more complicated than additive ones.
We set a scalar parameter of the model C2 as the unknown parameter to be estimated, $\theta \in \mathbb{R}$.
As the data-generating value of $\theta$ is unknown, we used a value suggested in the paper and the codes \citep{gamellaCausalChambersRealworld2025} as a reference to compute the estimation error of $\theta$.

\paragraph{\textsc{Light tunnel}}

We used a dataset comprising real-world measurements from a light tunnel made of light sources, linear polarizers, and a light sensor \citep{gamellaCausalChambersRealworld2025}.
We use the data to formulate a task to predict RGB images taken by the sensor, $y \in [0,1]^{3 \times 100 \times 100}$, from the light source configuration and the polarizer angles, $x \in \mathbb{R}^{5}$.
As an incomplete scientific model, we use the model ``F3'' presented in \citet{gamellaCausalChambersRealworld2025}, which models the effect of the polarizers in a frequency-dependent way.
We treat nine parameters in the model F3 as unknown parameters to be estimated, i.e., $\theta \in \mathbb{R}^9$.
% The scale of $\theta$ is inherently unidentifiable, so we assess the estimation error by cosine similarity.
Our hybrid model transforms the output of the scientific model with a convolutional neural net; so its architecture depends on function compositions more complicated than mere additive combinations.

\subsection{Training Methods}

We examined three baselines: \texttt{erm}, \texttt{p-reg}, and \texttt{f-reg} as introduced below.
Meanwhile, the proposed SAM-based training was implemented with the original SAM \citep{foretSharpnessawareMinimizationEfficiently2021}, adaptive SAM \citep{kwonASAMAdaptiveSharpnessaware2021}, and Fisher SAM \citep{kimFisherSAMInformation2022}; we will refer to the three as \texttt{sam}, \texttt{asam}, and \texttt{fsam}, respectively.

\paragraph{Empirical risk minimization (\texttt{erm})}
We consider the empirical risk minimization without any particular hybrid model regularization on the machine learning part's parameters $\phi$, i.e., to minimize the prediction error.
We applied a slight weight decay just to stabilize the optimization.

\paragraph{Parameter regularization (\texttt{p-reg})}
As a straightforward regularizer, we use the parameter $\ell_2$ norm
\begin{equation*}
    R_\mathrm{p}(\phi) = \lambda_\mathrm{p} \Vert \phi \Vert_2^2.
\end{equation*}
The regularized risk is then $L_S(\theta,\phi) + R_\mathrm{p}(\phi)$, and $\lambda_\mathrm{p} > 0$ is a tunable hyperparameter.

\paragraph{Functional regularization (\texttt{f-reg})}
We employ the idea to suppress the machine learning part's flexibility as a function by penalizing its functional norm \citep{yinAugmentingPhysicalModels2021} or the discrepancy between the overall model and the scientific-model-only reduction \citep{takeishiPhysicsintegratedVariationalAutoencoders2021}.
The exact form of the regularizer depends on model architectures and must be designed accordingly.
For the \textsc{pendulum time-series}, \textsc{Duffing oscillator}, and \textsc{reaction-diffusion} tasks, where $f_\theta$ and $g_\phi$ are additively combined, we define the regularizer as the minimization of the function norm \citep{yinAugmentingPhysicalModels2021}:
\begin{equation*}
    R_\mathrm{f}(\phi) = \frac{\lambda_\mathrm{f}}{|X|} \sum_{x \in X} \Vert g_\phi(x) \Vert_2^2,
\end{equation*}
where $X$ is a set of $x$ drawn from data or some distribution.
For the \textsc{pendulum images}, \textsc{wind tunnel}, and \textsc{light tunnel} tasks, as the models are not additive, we follow the idea of \citet{takeishiPhysicsintegratedVariationalAutoencoders2021}; we define the regularizer as the minimization of
\begin{equation*}
    R_\mathrm{f}(\phi) = \frac{\lambda_\mathrm{f}}{|X|} \sum_{x \in X} \Vert h_{\theta,\phi}(x) - h'_{\theta,\phi'}(x) \Vert_2^2,
\end{equation*}
where $h'_{\theta,\phi'}$ is a scientific-model-only reduction of the hybrid model; for instance, for a hybrid model of the form $h_{\theta,\phi}(x) = g_\phi(x, f_\theta(x))$, we define $h'_{\theta,\phi'}(x) = A f_\theta(x)$ with a linear map $A$, where $\phi' = \{ A \}$.
Such a linear map is required when the dimensionalities of $z$ and $y$ are different; when they are in the same ambient space, we simply fix $A$ as the identity map.
For both definitions, the coefficient $\lambda_\mathrm{f} > 0$ is a hyperparameter.

\paragraph{SAM family (\texttt{sam}, \texttt{asam}, \texttt{fsam})}
We apply SAM to the learning of hybrid models as motivated in \cref{method}.
Instead of fixing the learning rates as in \cref{alg:main}, we use adaptive learning rate by Adam.
We try three different ways to compute the ascend direction $\epsilon^*_\phi$: the standard SAM \citep{foretSharpnessawareMinimizationEfficiently2021}, adaptive SAM \citep{kwonASAMAdaptiveSharpnessaware2021}, and Fisher SAM \citep{kimFisherSAMInformation2022}.
We will refer to these three variants as \texttt{sam}, \texttt{asam}, and \texttt{fsam}, respectively.
The tunable hyperparameter is the perturbation radius $\rho_\phi$.

\subsection{Results}

\begin{table}[t]
    \caption{Estimation errors of $\theta$ and prediction errors of $y$. For \textsc{light tunnel}, we report the cosine similarity between the estimated and reference $\theta$'s because the scale of $\theta$ is structurally unidentifiable; otherwise the root mean squared errors are reported. The best and the second best results are \textbf{bolded} and \underline{underlined}, respectively. Mean and SD over 5 runs with different random seeds are reported.}
    \vspace*{-3pt}
    \label{tab:main_result}
    \centering
    \footnotesize
    \renewcommand{\arraystretch}{0.95}
    \setlength{\tabcolsep}{3pt}
    \newcommand{\rescell}[2]{\makecell{$#1$ {\scriptsize $\pm #2$}}}
    \newcommand{\taskcell}[1]{\multirow[c]{2}{*}[-0ex]{\makecell{#1}}}
    \begin{tabular}{clcccccc}
    
\toprule
& & \texttt{erm} & \texttt{p-reg} & \texttt{f-reg} & \texttt{sam} & \texttt{asam} & \texttt{fsam} \\
\midrule

\taskcell{\textsc{pendulum} \\ \textsc{time-series}} &

\footnotesize $\theta$-error \scriptsize ($\times 10^{-3}$) &
\rescell{19.9}{16.9} &
\rescell{\underline{2.55}}{0.10} &
\rescell{\mathbf{0.37}}{0.09} &
\rescell{4.00}{0.54} &
\rescell{3.35}{1.08} &
\rescell{3.74}{0.46} \\
& \footnotesize $y$-error \scriptsize ($\times 10^{-2}$) &
\rescell{9.11}{0.13} &
\rescell{8.00}{0.02} &
\rescell{43.1}{0.42} &
\rescell{8.21}{0.06} &
\rescell{8.11}{0.07} &
\rescell{8.20}{0.06} \\

\midrule

\taskcell{\hspace{-1em}\textsc{reaction} \\ \hspace{1em}\textsc{-diffusion}} &

\footnotesize $\theta$-error \scriptsize ($\times 10^{-4}$) &
\rescell{9.27}{0.69} &
\rescell{\mathbf{0.72}}{0.31} &
\rescell{\underline{1.26}}{0.66} &
\rescell{2.79}{1.97} &
\rescell{3.66}{1.67} &
\rescell{2.80}{1.98} \\
& \footnotesize $y$-error \scriptsize ($\times 10^{0}$) &
\rescell{1.11}{$<$0.01} &
\rescell{1.11}{$<$0.01} &
\rescell{1.11}{$<$0.01} &
\rescell{1.19}{0.01} &
\rescell{1.25}{0.02} &
\rescell{1.19}{0.01} \\

\midrule

\taskcell{\textsc{Duffing} \\ \textsc{oscillator}} &

\footnotesize $\theta$-error \scriptsize ($\times 10^{-2}$) &
\rescell{30.8}{2.25} &
\rescell{13.4}{2.41} &
\rescell{18.7}{3.67} &
\rescell{\mathbf{1.32}}{0.65} &
\rescell{\underline{3.22}}{1.75} &
\rescell{\mathbf{1.32}}{0.91} \\
& \footnotesize $y$-error \scriptsize ($\times 10^{-2}$) &
\rescell{6.74}{0.05} &
\rescell{6.45}{0.01} &
\rescell{6.75}{0.04} &
\rescell{6.43}{0.02} &
\rescell{6.46}{0.02} &
\rescell{6.43}{0.03} \\

\midrule

\taskcell{\textsc{pendulum} \\ \textsc{images}} &

\footnotesize $\theta$-error \scriptsize ($\times 10^{-2}$) &
\rescell{12.2}{5.77} &
\rescell{12.9}{2.79} &
\rescell{11.7}{5.66} &
\rescell{5.83}{3.86} &
\rescell{\underline{5.48}}{3.39} &
\rescell{\mathbf{5.04}}{4.25} \\
& \footnotesize $y$-error \scriptsize ($\times 10^{-0}$) &
\rescell{9.07}{0.14} &
\rescell{9.13}{0.27} &
\rescell{8.98}{0.12} &
\rescell{9.75}{0.30} &
\rescell{9.57}{0.15} &
\rescell{9.59}{0.29} \\

\midrule

\taskcell{\textsc{wind} \\ \textsc{tunnel}} &

\footnotesize $\theta$-error \scriptsize ($\times 10^{-2}$) &
\rescell{69.6}{0.18} &
\rescell{69.5}{0.29} &
\rescell{4.39}{0.09} &
\rescell{\underline{2.15}}{0.98} &
\rescell{2.39}{0.77} &
\rescell{\mathbf{1.90}}{0.67} \\
& \footnotesize $y$-error \scriptsize ($\times 10^{1}$) &
\rescell{5.31}{0.03} &
\rescell{5.30}{0.02} &
\rescell{5.28}{0.03} &
\rescell{5.11}{0.01} &
\rescell{5.09}{0.02} &
\rescell{5.09}{0.01} \\

\midrule

\taskcell{\textsc{light} \\ \textsc{tunnel}} &

\footnotesize $\theta$ cosine sim. $\uparrow$ &
\rescell{0.76}{0.27} &
\rescell{0.87}{0.23} &
\rescell{\underline{0.96}}{0.01} &
\rescell{\mathbf{0.98}}{0.01} &
\rescell{\mathbf{0.98}}{0.01} &
\rescell{\mathbf{0.98}}{0.01} \\
& \footnotesize $y$-error \scriptsize ($\times 10^{0}$) &
\rescell{4.16}{0.06} &
\rescell{4.29}{0.07} &
\rescell{4.21}{0.05} &
\rescell{5.03}{0.14} &
\rescell{8.22}{0.30} &
\rescell{4.61}{0.02} \\

\bottomrule
        
    \end{tabular}
    \vspace*{-2ex}
    
\end{table}

\paragraph{Prediction and estimation errors (Table~\ref*{tab:main_result})}
\Cref{tab:main_result} reports the test prediction error ($y$-error) and the estimation error of $\theta$ ($\theta$-error).
The magnitude of the values can be guessed in comparison to the values in the \texttt{erm} column.
\textbf{The SAM-based methods achieved better or comparable performance in all the tasks, although they do not assume specific hybrid model architectures}.
SAM family shows slightly worse $y$-errors in some cases, but the degradation is mostly qualitatively insignificant, and if needed we can re-train or finetune $\phi$ after fixing $\theta$ to improve the $y$-error.
More specifically:
\begin{itemize}[leftmargin=*,topsep=0pt,itemsep=0pt]
    \item For \textsc{pendulum time-series} and \textsc{reaction-diffusion}, the method of \citet{yinAugmentingPhysicalModels2021} was implemented as \texttt{f-reg}.
    It should work well because the hybrid models are in the additive form, as already confirmed in \citep{yinAugmentingPhysicalModels2021}.
    In these tasks, the performance of the SAM-based methods is slightly worse than \texttt{p-reg} and \texttt{f-reg} but is significantly improved from \texttt{erm}.
    Interestingly, for the \textsc{reaction-diffusion} task, the mere $\ell_2$ regularization (\texttt{p-reg}) worked better than \texttt{f-reg}.
    \item For \textsc{pendulum images}, \textsc{wind tunnel}, and \textsc{light tunnel}, where the models are non-additive, the baseline methods did not work well.
    It is not surprising because the baseline regularizers need special design for each architecture, which requires careful tuning to achieve good $\theta$ estimation.
    \item The result for \textsc{Duffing oscillator} is interesting. Although the model was additive, the function norm-based \texttt{f-reg} did not estimate $\theta$ well, while the SAM-based methods were successful.
\end{itemize}

\begin{figure}[t]
    \begin{minipage}[t]{0.52\textwidth}
        \vspace*{0pt}
        \centering
        \setlength{\tabcolsep}{0.5pt}
        \renewcommand{\arraystretch}{1}
        \newcommand{\imgw}{0.126\textwidth}
        \newcommand{\vlabel}[1]{\raisebox{0.055\textwidth}{\scriptsize $#1$}}
        {\small\begin{tabular}{cccccccccccc}
            \rotatebox{90}{\hspace{5pt}truth}\hspace{5pt} &
            &
            \includegraphics[width=\imgw]{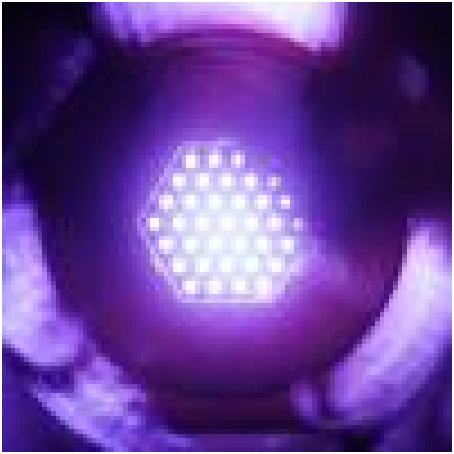} &
            \includegraphics[width=\imgw]{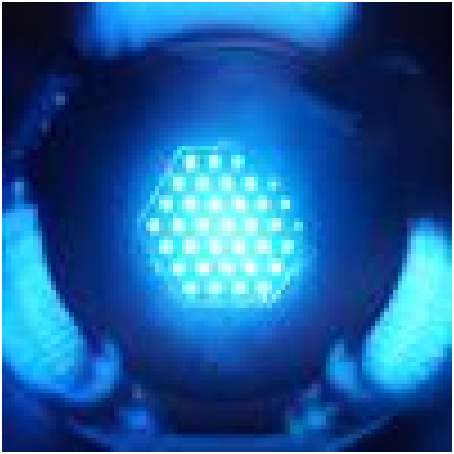} &
            \includegraphics[width=\imgw]{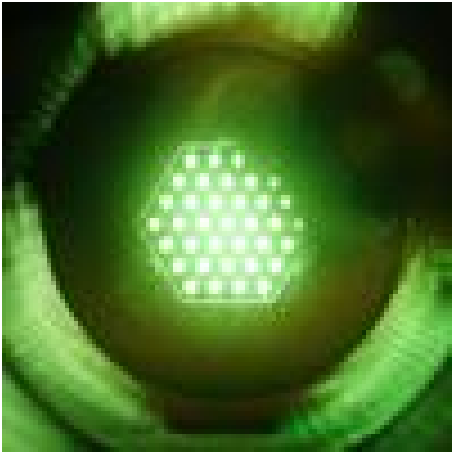} &
            \includegraphics[width=\imgw]{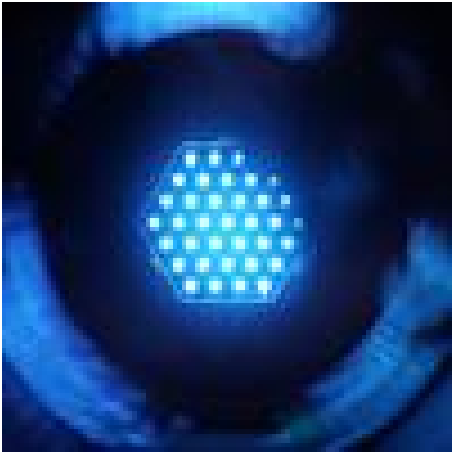} &
            \includegraphics[width=\imgw]{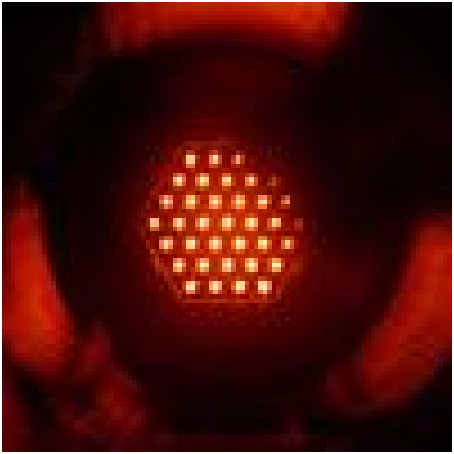} &
            \includegraphics[width=\imgw]{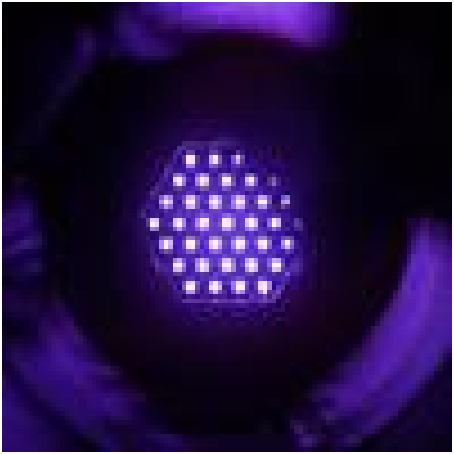} &
            \includegraphics[width=\imgw]{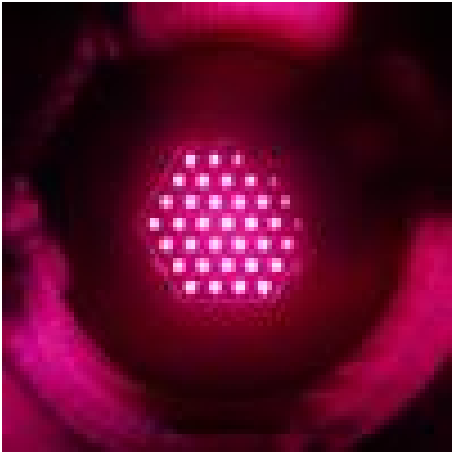} \\[5pt]

            \multirow{2}{*}{\rotatebox{90}{\texttt{erm}}\hspace{5pt}} &
            \vlabel{\tilde{z}} &
            \includegraphics[width=\imgw]{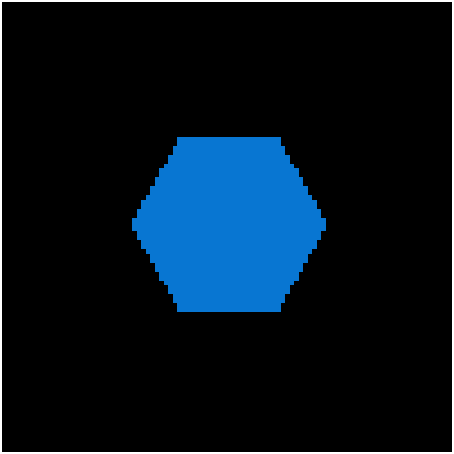} &
            \includegraphics[width=\imgw]{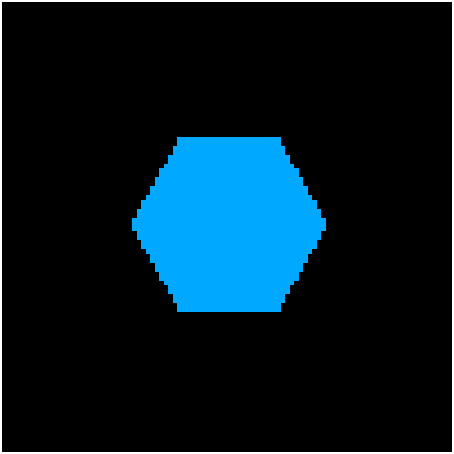} &
            \includegraphics[width=\imgw]{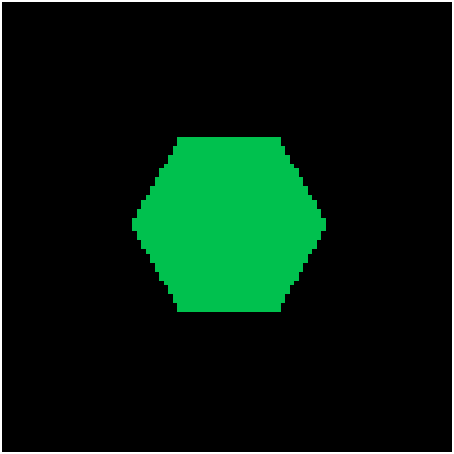} &
            \includegraphics[width=\imgw]{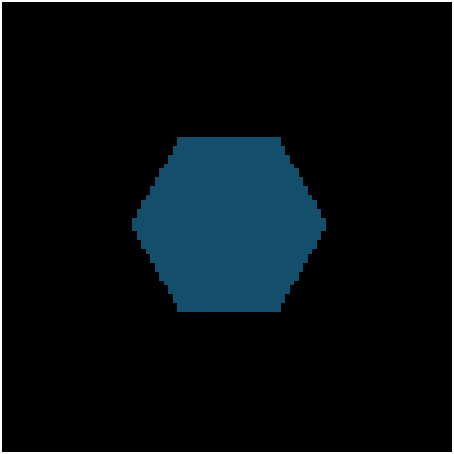} &
            \includegraphics[width=\imgw]{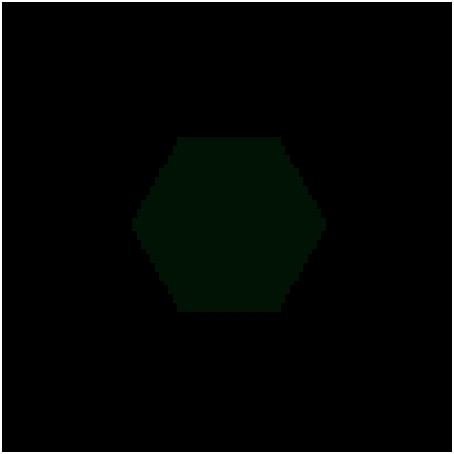} &
            \includegraphics[width=\imgw]{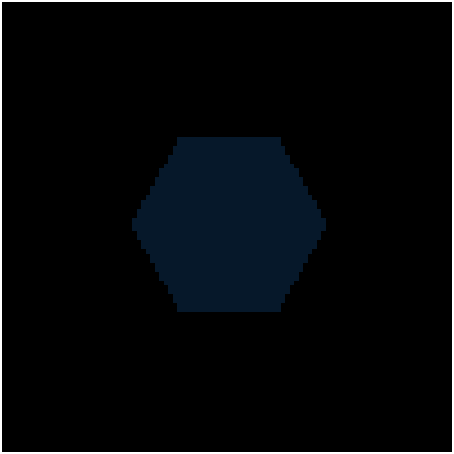} &
            \includegraphics[width=\imgw]{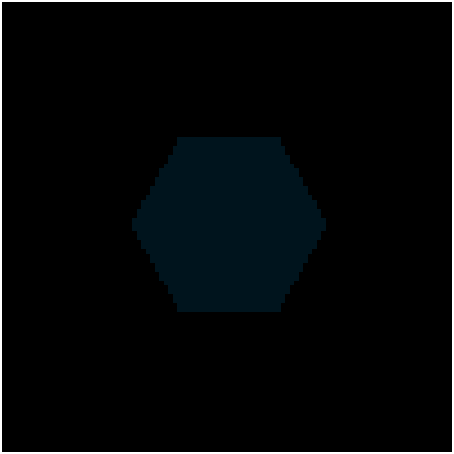} \\

            &
            \vlabel{\tilde{y}} &
            \includegraphics[width=\imgw]{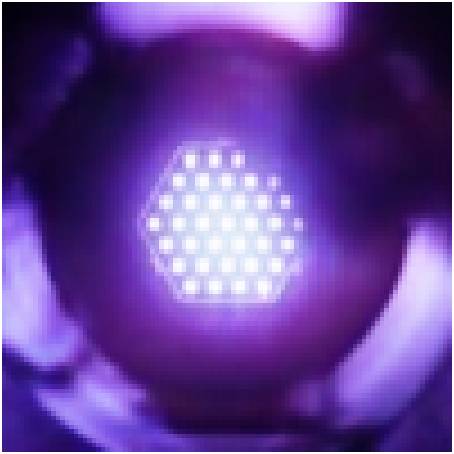} &
            \includegraphics[width=\imgw]{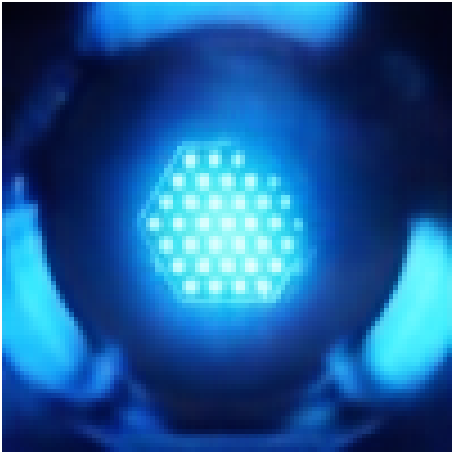} &
            \includegraphics[width=\imgw]{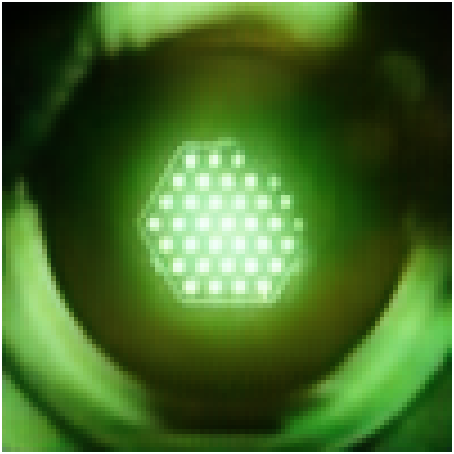} &
            \includegraphics[width=\imgw]{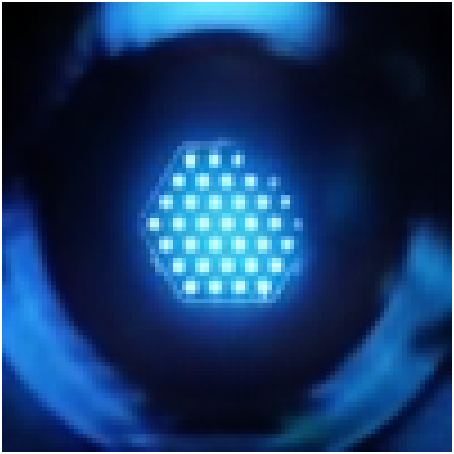} &
            \includegraphics[width=\imgw]{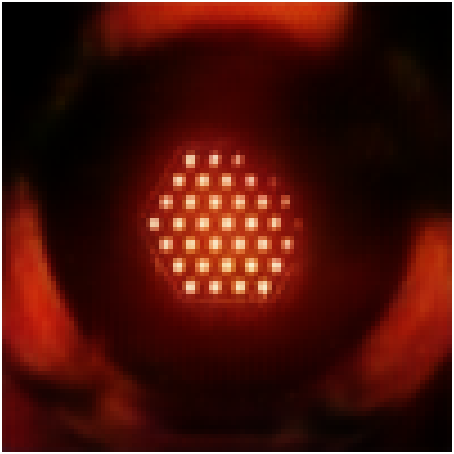} &
            \includegraphics[width=\imgw]{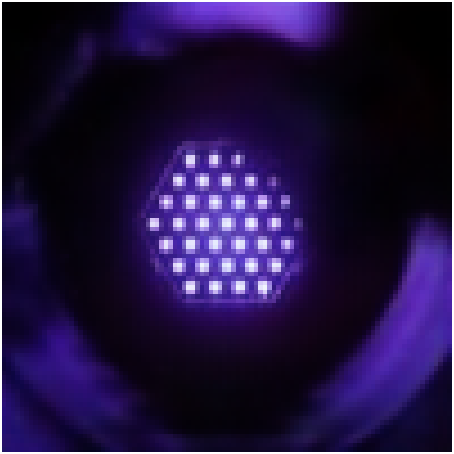} &
            \includegraphics[width=\imgw]{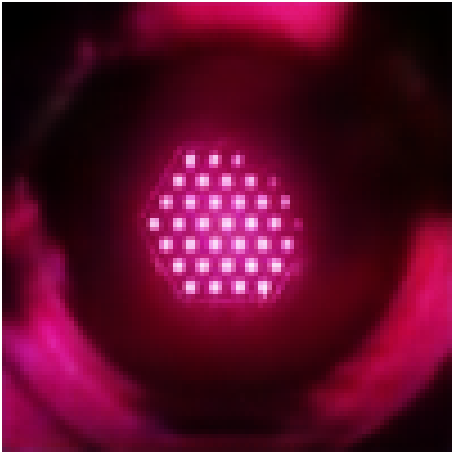} \\[5pt]

            \multirow{2}{*}{\rotatebox{90}{\texttt{fsam}}\hspace{5pt}} &
            \vlabel{\tilde{z}} &
            \includegraphics[width=\imgw]{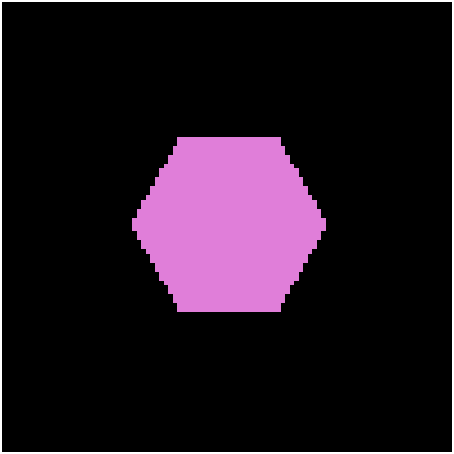} &
            \includegraphics[width=\imgw]{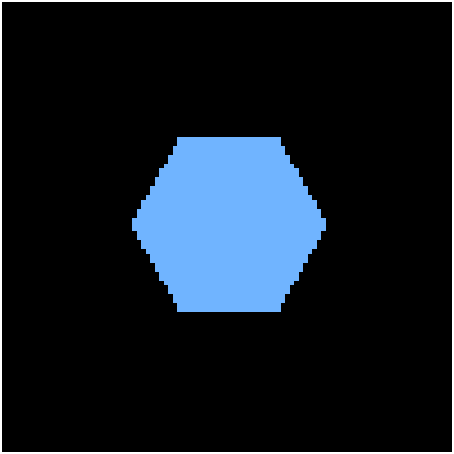} &
            \includegraphics[width=\imgw]{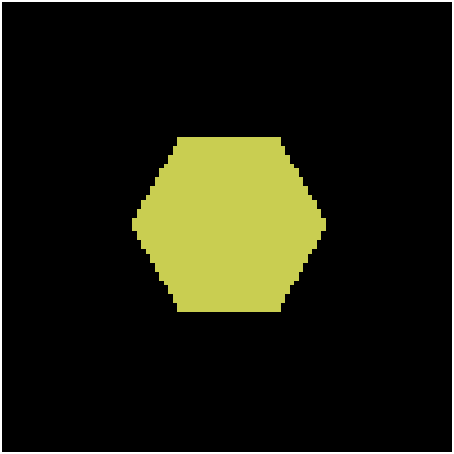} &
            \includegraphics[width=\imgw]{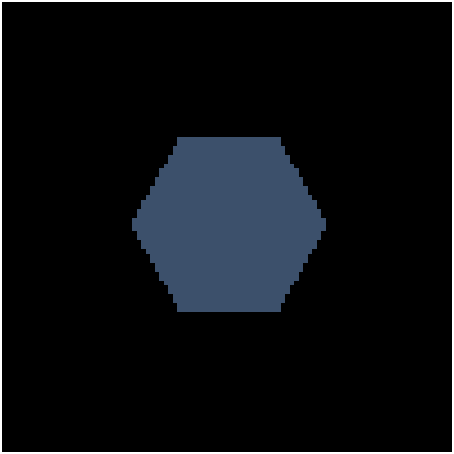} &
            \includegraphics[width=\imgw]{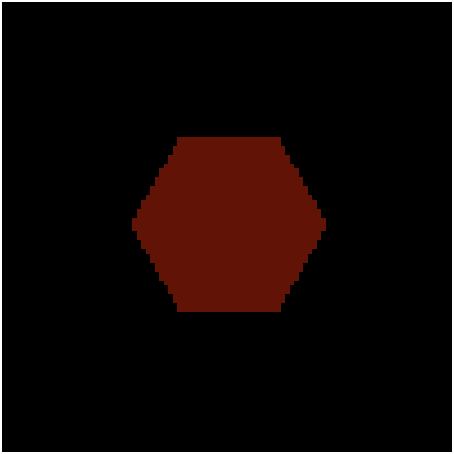} &
            \includegraphics[width=\imgw]{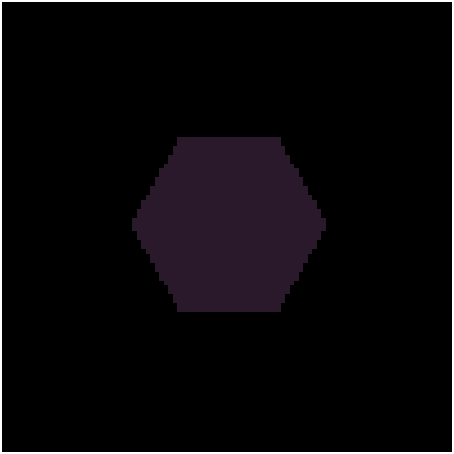} &
            \includegraphics[width=\imgw]{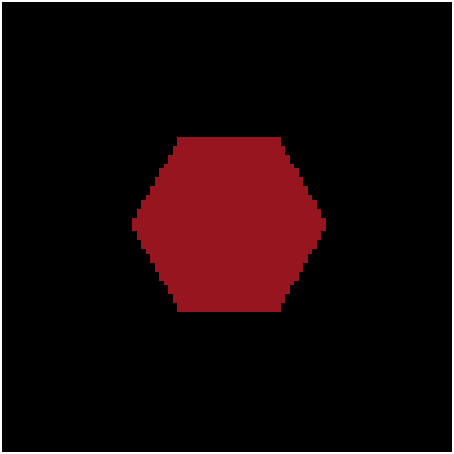} \\

            &
            \vlabel{\tilde{y}} &
            \includegraphics[width=\imgw]{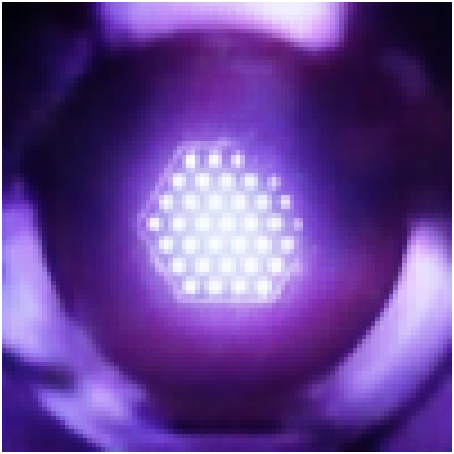} &
            \includegraphics[width=\imgw]{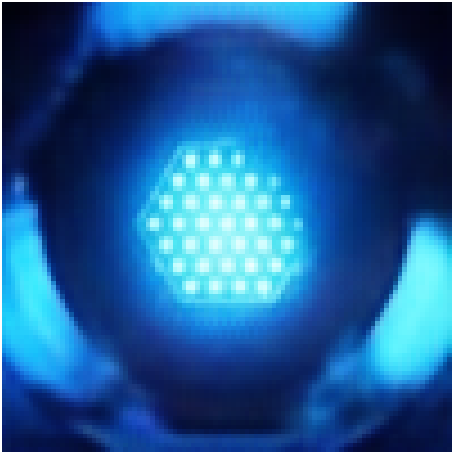} &
            \includegraphics[width=\imgw]{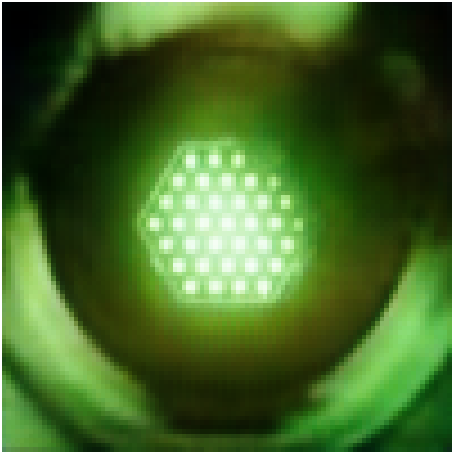} &
            \includegraphics[width=\imgw]{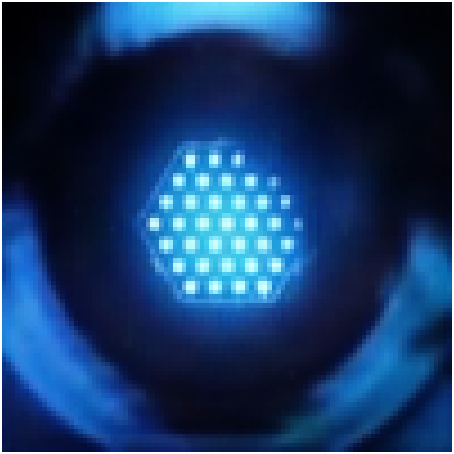} &
            \includegraphics[width=\imgw]{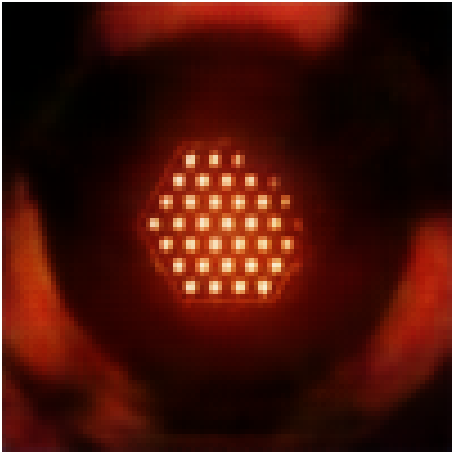} &
            \includegraphics[width=\imgw]{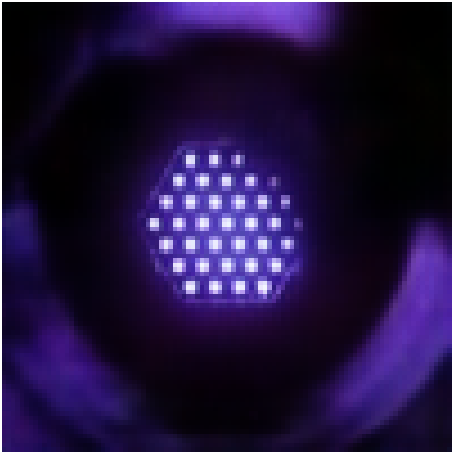} &
            \includegraphics[width=\imgw]{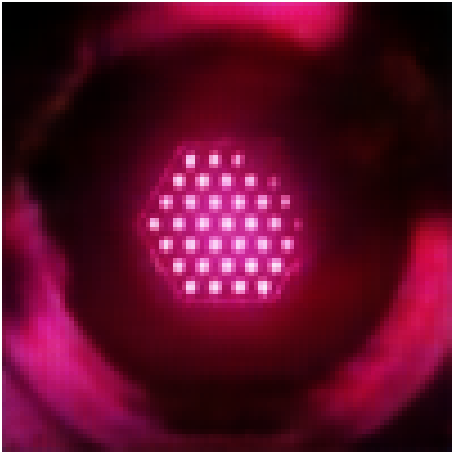} \\

        \end{tabular}}
        \caption{Data and predictions in the \textsc{light tunnel} task: (\emph{top row}) ground truth images; (\emph{middle two rows}) predictions from the \texttt{erm} model; and (\emph{bottom two rows}) predictions from the \texttt{fsam} model. For each model, the upper row shows the predictions $\tilde{z}$ from the scientific model part (model F2), and the lower row shows the full predictions $\tilde{y}$.}
        \label{fig:light_tunnel_summary}
    \end{minipage}
    \hfill
    \begin{minipage}[t]{0.44\textwidth}
        \vspace*{0pt}
        \centering
        \includegraphics[trim={0 10pt 0 0},clip,width=0.49\linewidth]{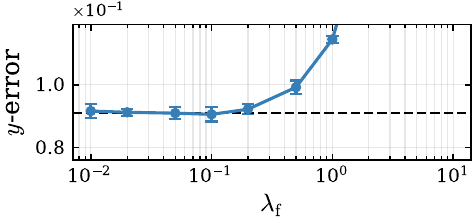}
        \hfill
        \includegraphics[trim={10pt 10pt 0 0},clip,width=0.47\linewidth]{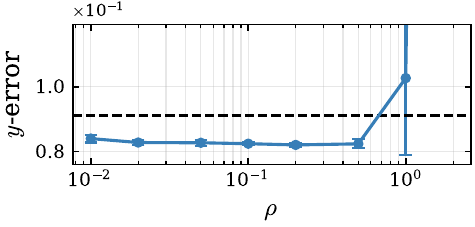}
        \\
        \includegraphics[trim={0 0 0 0},clip,width=0.49\linewidth]{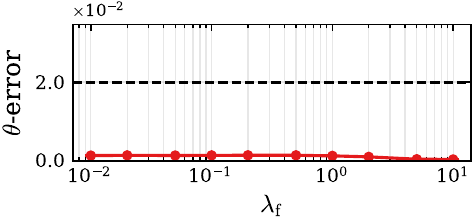}
        \hfill
        \includegraphics[trim={10pt 0 0 0},clip,width=0.47\linewidth]{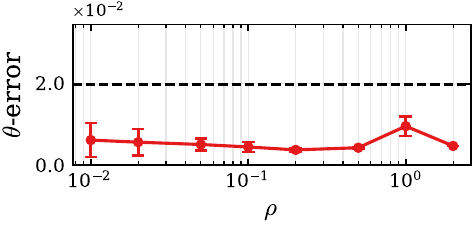}
        \\[-1pt]
        \scriptsize{(a) \textsc{pendulum time-series}: (\emph{left}) \texttt{freg} and (\emph
        {right}) \texttt{fsam}}
        \\[1ex]
        \includegraphics[trim={0 10pt 0 0},clip,width=0.49\linewidth]{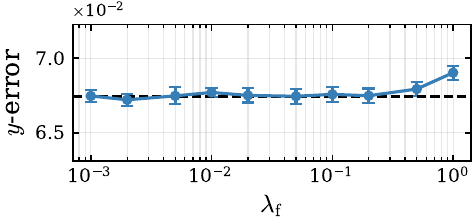}
        \hfill
        \includegraphics[trim={10pt 10pt 0 0},clip,width=0.47\linewidth]{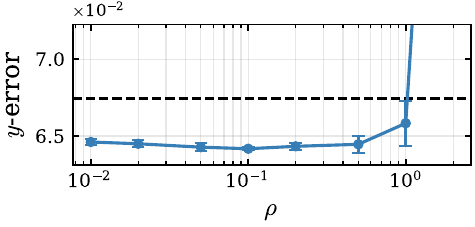}
        \\
        \includegraphics[trim={0 0 0 0},clip,width=0.49\linewidth]{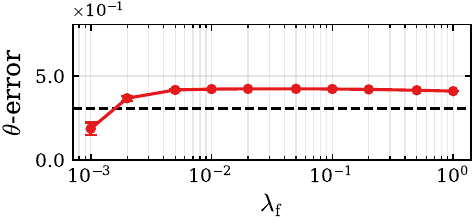}
        \hfill
        \includegraphics[trim={10pt 0 0 0},clip,width=0.47\linewidth]{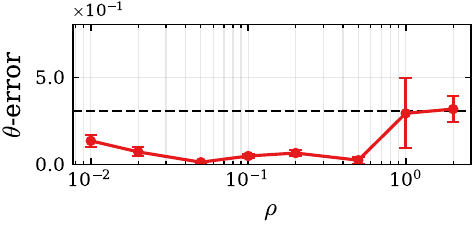}
        \\[-1pt]
        \scriptsize{(b) \textsc{Duffing oscillator}: (\emph{left}) \texttt{freg} and (\emph
        {right}) \texttt{fsam}}
        \caption{Effects of the configuration of hyperparameters: $\lambda_\mathrm{f}$ for \texttt{f-reg} and $\rho$ for \texttt{fsam}. The dashed lines represent the average performance of \texttt{erm}. Full results are presented in the appendix.}
        \label{fig:hyperparam_sensitivity}
    \end{minipage}
    \vspace*{-2.5ex}
\end{figure}

\paragraph{Model behaviors (Figure~\ref*{fig:light_tunnel_summary})}
\cref{fig:light_tunnel_summary} exemplifies the predictions in the \textsc{light tunnel} task, where for each method, the upper row is the prediction by the scientific model part, $\tilde{z} := f_{\tilde\theta}(x)$, whereas the lower row is the full prediction by the hybrid model, $\tilde{y} = h_{\tilde\theta,\tilde\phi}(x)$.
The light colors predicted as $\tilde{z}$ by the \texttt{erm} model sometimes deviate from the truth, implying the inadequate use of the scientific model due to the misestimation of $\theta$.
The colors in $\tilde{z}$ of the \texttt{fsam} model align well with the truth.

\paragraph{Hyperparameter sensitivity (Figure~\ref*{fig:hyperparam_sensitivity})}
We investigate the sensitivity of the performance to the configuration of the hyperparameters: $\lambda_\mathrm{p}$, $\lambda_\mathrm{f}$, and $\rho_\phi$ for the \texttt{p-reg}, \texttt{f-reg}, and \{\texttt{sam}, \texttt{asam}, \texttt{fsam}\}, respectively.
We show a part of the results in \cref{fig:hyperparam_sensitivity}; the remaining results are in Appendix~\ref{appendix:results}.
As a reference we also show the performance by the \texttt{erm} model with dashed lines.
In terms of the $\theta$-estimation error, the performance is not very sensitive to the hyperparameter value within the standard range of $\rho$ \citep{foretSharpnessawareMinimizationEfficiently2021,kwonASAMAdaptiveSharpnessaware2021,kimFisherSAMInformation2022}.
On the other hand, the $y$-prediction error gets significantly worse than the reference by \texttt{erm} when the regularization effect becomes too strong.

\paragraph{Effect of SAM modification (Table~\ref*{table:hybrid_sam_effect})}
\begin{wraptable}[9]{r}[0pt]{.43\textwidth}
    \centering
    \vspace*{-2ex}
    \caption{$\theta$-error by ordinary Fisher SAM}
    \vspace*{2pt}
    \label{table:hybrid_sam_effect}
    \renewcommand{\arraystretch}{1}
    \setlength{\tabcolsep}{3pt}
    \newcommand{\taskcell}[1]{\makecell{#1}}
    \newcommand{\rescell}[2]{\makecell{$#1$ {\scriptsize $\pm #2$}}}
    {\footnotesize\begin{tabular}{llcc}
        \toprule
        & & \texttt{fsam} & ordinary \\
        \midrule
        \taskcell{\textsc{pend t-s}} &
        ($\times 10^{-3}$) &
        \rescell{\mathbf{3.7}}{0.5} &
        \rescell{37}{7.1} \\
        \taskcell{\textsc{react-diff}} &
        ($\times 10^{-4}$) &
        \rescell{\mathbf{2.8}}{2.0} &
        \rescell{\mathbf{2.8}}{1.1} \\
        \taskcell{\textsc{duffing}} &
        ($\times 10^{-2}$) &
        \rescell{\mathbf{1.3}}{0.9} &
        \rescell{1.4}{1.0} \\
        \taskcell{\textsc{pend image}} &
        ($\times 10^{-2}$) &
        \rescell{5.0}{4.3} &
        \rescell{\mathbf{3.3}}{2.6} \\
        \taskcell{\textsc{wind}} &
        ($\times 10^{-2}$) &
        \rescell{\mathbf{1.9}}{0.7} &
        \rescell{5.4}{1.5} \\
        \taskcell{\textsc{light}} &
        (cos-sim $\uparrow$) &
        \rescell{\mathbf{.98}}{.01} &
        \rescell{.97}{.02} \\
        \bottomrule
    \end{tabular}}
\end{wraptable}
In \cref{alg:main}, we perturb only $\phi$ and leave $\theta$ unperturbed, whereas ordinary SAM methods perturb all the parameters.
We examined the practical effect of such a modification.
\cref{table:hybrid_sam_effect} summarizes the $\theta$-estimation error by our application of Fisher SAM (perturbing only $\phi$; \texttt{fsam}) and the ordinary Fisher SAM (perturbing $\phi$ and $\theta$).
For the \textsc{pendulum time-series} and \textsc{wind tunnel} tasks, the modified method, \texttt{fsam}, achieves clearly smaller error than the ordinary case.
The effect of limiting the perturbation to $\phi$ appears marginal in some cases, as we anticipated earlier.
For the \textsc{pendulum images} task, the ordinary Fisher SAM performs slightly better, though not significantly so.
The neural net for this task has a relatively large number of parameters for encoding and decoding the image time series, for which the benefit of limiting the perturbation might have been small and lost in the random variation.

\section{Conclusion}

In this paper, we suggest focusing on loss flatness in hybrid modeling, the combination of scientific mathematical models and machine learning models.
The underlying idea is that the machine learning part of a hybrid model should be kept as simple as possible for the maximum use of the scientific model part.
We have hypothesized that SAM is useful to achieve such a simple model.
We have empirically validated the effectiveness of the idea on different tasks using synthetic and real-world data.
An overall observation from the experiments is that SAM is indeed useful for identifying unknown parameters of the scientific model part, despite not requiring specific hybrid model architectures.
The SAM-based learning method for hybrid models works in an architecture-agnostic manner, unlike the existing regularization methods that require specific design depending on the model architecture.

An important limitation is that unless one has specific assumptions about the data-generating process, the proposed learning method does not necessarily identify \emph{the correct} value of the unknown scientific parameters, as is the case with the existing methods for hybrid modeling.
It is something essential as there is often no such thing like the correct scientific parameter; the ``correctness'' solely depends on one's beliefs about the data-generating process and the principles of model selection.
To overcome the fundamental limitation, one has to develop a specific method and theory for particular hybrid model architectures.
Establishing a rigorous and universal theory of identification for general hybrid models is an interesting yet challenging open problem.

\section*{Acknowledgements}

This work was supported by JST PRESTO JPMJPR24T6 and JSPS JP25H01454.

% \section*{Impact Statement}

% This work proposes a method for learning hybrid models that combine scientific mathematical models with machine learning models.
% While it is difficult to anticipate broader societal consequences at this stage of fundamental research, hybrid modeling is often motivated by the need for more interpretable and reliable machine learning.
% If deployed in critical settings without sufficient validation, however, hybrid models could still lead to negative outcomes.
% Therefore, careful evaluation and responsible use are important.

\bibliography{main,books}

@book{dupuis_ellis_1997,
    author={Paul Dupuis and Richard S. Ellis},
    title={A Weak Convergence Approach to the Theory of Large Deviations},
    publisher={Wiley},
    year={1997}
}

@inproceedings{ajayAugmentingPhysicalSimulators2018,
  title = {Augmenting Physical Simulators with Stochastic Neural Networks: {{Case}} Study of Planar Pushing and Bouncing},
  booktitle = {Proceedings of the 2018 {{IEEE}}/{{RSJ International Conference}} on {{Intelligent Robots}} and {{Systems}}},
  author = {Ajay, Anurag and Wu, Jiajun and Fazeli, Nima and Bauza, Maria and Kaelbling, Leslie P. and Tenenbaum, Joshua B. and Rodriguez, Alberto},
  year = 2018,
  pages = {3066--3073}
}

@article{anstett-collinPrioriIdentifiabilityOverview2020,
  title = {A Priori Identifiability: {{An}} Overview on Definitions and Approaches},
  author = {{Anstett-Collin}, F. and {Denis-Vidal}, L. and Mill{\'e}rioux, G.},
  year = 2020,
  journal = {Annual Reviews in Control},
  volume = {50},
  pages = {139--149}
}

@inproceedings{arikInterpretableSequenceLearning2020,
  title = {Interpretable Sequence Learning for {{COVID-19}} Forecasting},
  booktitle = {Advances in {{Neural Information Processing Systems}} 33},
  author = {Ar{\i}k, Sercan {\"O}. and Li, Chun-Liang and Yoon, Jinsung and Sinha, Rajarishi and Epshteyn, Arkady and Le, Long T. and Menon, Vikas and Singh, Shashank and Zhang, Leyou and Yoder, Nate and Nikoltchev, Martin and Sonthalia, Yash and Nakhost, Hootan and Kanal, Elli and Pfister, Tomas},
  year = 2020,
  pages = {18807--18818}
}

@article{baldassiUnreasonableEffectivenessLearning2016,
  title = {Unreasonable Effectiveness of Learning Neural Networks: {{From}} Accessible States and Robust Ensembles to Basic Algorithmic Schemes},
  author = {Baldassi, Carlo and Borgs, Christian and Chayes, Jennifer T. and Ingrosso, Alessandro and Lucibello, Carlo and Saglietti, Luca and Zecchina, Riccardo},
  year = 2016,
  journal = {Proceedings of the National Academy of Sciences},
  volume = {113},
  number = {48},
  pages = {E7655-E7662},
  publisher = {Proceedings of the National Academy of Sciences}
}

@inproceedings{chaudhariEntropySGDBiasingGradient2017,
  title = {Entropy-{{SGD}}: Biasing Gradient Descent into Wide Valleys},
  booktitle = {Proceedings of the 5th {{International Conference}} on {{Learning Representations}}},
  author = {Chaudhari, Pratik and Choromanska, Anna and Soatto, Stefano and LeCun, Yann and Baldassi, Carlo and Borgs, Christian and Chayes, Jennifer and Sagun, Levent and Zecchina, Riccardo},
  year = 2017
}

@article{claesHybridAdditiveModeling2025,
  title = {Hybrid Additive Modeling with Partial Dependence for Supervised Regression and Dynamical Systems Forecasting},
  author = {Claes, Yann and {Huynh-Thu}, V{\^a}n Anh and Geurts, Pierre},
  year = 2025,
  journal = {Machine Learning},
  volume = {114},
  number = {3},
  pages = {72}
}

@article{cohrsCausalHybridModeling2024,
  title = {Causal Hybrid Modeling with Double Machine Learning---Applications in Carbon Flux Modeling},
  author = {Cohrs, Kai-Hendrik and Varando, Gherardo and Carvalhais, Nuno and Reichstein, Markus and {Camps-Valls}, Gustau},
  year = 2024,
  journal = {Machine Learning: Science and Technology},
  volume = {5},
  number = {3},
  pages = {035021},
  publisher = {arXiv}
}

@inproceedings{dinhSharpMinimaCan2017,
  title = {Sharp Minima Can Generalize for Deep Nets},
  booktitle = {Proceedings of the 34th {{International Conference}} on {{Machine Learning}}},
  author = {Dinh, Laurent and Pascanu, Razvan and Bengio, Samy and Bengio, Yoshua},
  year = 2017,
  pages = {1019--1028}
}

@inproceedings{dziugaiteComputingNonvacuousGeneralization2017,
  title = {Computing Nonvacuous Generalization Bounds for Deep (Stochastic) Neural Networks with Many More Parameters than Training Data},
  booktitle = {Proceedings of the 33rd {{Conference}} on {{Uncertainty}} in {{Artificial Intelligence}}},
  author = {Dziugaite, Gintare Karolina and Roy, Daniel M.},
  year = 2017
}

@inproceedings{foretSharpnessawareMinimizationEfficiently2021,
  title = {Sharpness-Aware Minimization for Efficiently Improving Generalization},
  booktitle = {Proceedings of the 9th {{International Conference}} on {{Learning Representations}}},
  author = {Foret, Pierre and Kleiner, Ariel and Mobahi, Hossein and Neyshabur, Behnam},
  year = 2021
}

@article{forssellCombiningSemiphysicalNeural1997,
  title = {Combining Semi-Physical and Neural Network Modeling: {{An}} Example of Its Usefulness},
  author = {Forssell, U. and Lindskog, P.},
  year = 1997,
  journal = {IFAC Proceedings Volumes},
  volume = {30},
  number = {11},
  pages = {767--770}
}

@article{gamellaCausalChambersRealworld2025,
  title = {Causal Chambers as a Real-World Physical Testbed for {{AI}} Methodology},
  author = {Gamella, Juan L. and Peters, Jonas and B{\"u}hlmann, Peter},
  year = 2025,
  journal = {Nature Machine Intelligence},
  volume = {7},
  pages = {107--118}
}

@article{gaoSimtorealSoftRobots2024,
  title = {Sim-to-Real of Soft Robots with Learned Residual Physics},
  author = {Gao, Junpeng and Michelis, Mike Y. and Spielberg, Andrew and Katzschmann, Robert K.},
  year = 2024,
  journal = {IEEE Robotics and Automation Letters},
  volume = {9},
  number = {10},
  pages = {8523--8530}
}

@article{giampiccoloRobustParameterEstimation2024,
  title = {Robust Parameter Estimation and Identifiability Analysis with Hybrid Neural Ordinary Differential Equations in Computational Biology},
  author = {Giampiccolo, Stefano and Reali, Federico and Fochesato, Anna and Iacca, Giovanni and Marchetti, Luca},
  year = 2024,
  journal = {npj Systems Biology and Applications},
  volume = {10},
  number = {1},
  pages = {139}
}

@inproceedings{gravesPracticalVariationalInference2011a,
  title = {Practical Variational Inference for Neural Networks},
  booktitle = {Advances in {{Neural Information Processing Systems}} 24},
  author = {Graves, Alex},
  year = 2011,
  pages = {2348--2356}
}

@article{guillaumeIntroductoryOverviewIdentifiability2019,
  title = {Introductory Overview of Identifiability Analysis: {{A}} Guide to Evaluating Whether You Have the Right Type of Data for Your Modeling Purpose},
  author = {Guillaume, Joseph H.A. and Jakeman, John D. and {Marsili-Libelli}, Stefano and Asher, Michael and Brunner, Philip and Croke, Barry and Hill, Mary C. and Jakeman, Anthony J. and Keesman, Karel J. and Razavi, Saman and Stigter, Johannes D.},
  year = 2019,
  journal = {Environmental Modelling \& Software},
  volume = {119},
  pages = {418--432}
}

@inproceedings{haussmannLearningPartiallyKnown2021,
  title = {Learning Partially Known Stochastic Dynamics with Empirical {{PAC Bayes}}},
  booktitle = {Proceedings of the 24th {{International Conference}} on {{Artificial Intelligence}} and {{Statistics}}},
  author = {Hau{\ss}mann, Manuel and Gerwinn, Sebastian and Look, Andreas and Rakitsch, Barbara and Kandemir, Melih},
  year = 2021,
  pages = {478--486}
}

@inproceedings{heidenNeuralSimAugmentingDifferentiable2021,
  title = {{{NeuralSim}}: {{Augmenting}} Differentiable Simulators with Neural Networks},
  booktitle = {Proceedings of 2021 {{IEEE International Conference}} on {{Robotics}} and {{Automation}}},
  author = {Heiden, Eric and Millard, David and Coumans, Erwin and Sheng, Yizhou and Sukhatme, Gaurav S.},
  year = 2021,
  pages = {9474--9481}
}

@inproceedings{hintonAutoencodersMinimumDescription1993,
  title = {Autoencoders, Minimum Description Length and {{Helmholtz}} Free Energy},
  booktitle = {Advances in {{Neural Information Processing Systems}} 6},
  author = {Hinton, Geoffrey E and Zemel, Richard},
  year = 1993,
  pages = {3--10}
}

@inproceedings{hintonKeepingNeuralNetworks1993,
  title = {Keeping the Neural Networks Simple by Minimizing the Description Length of the Weights},
  booktitle = {Proceedings of the 6th {{Annual Conference}} on {{Computational Learning Theory}}},
  author = {Hinton, Geoffrey E. and {van Camp}, Drew},
  year = 1993,
  pages = {5--13}
}

@article{hochreiterFlatMinima1997,
  title = {Flat Minima},
  author = {Hochreiter, Sepp and Schmidhuber, J{\"u}rgen},
  year = 1997,
  journal = {Neural Computation},
  volume = {9},
  number = {1},
  pages = {1--42}
}

@inproceedings{holtAutomaticallyLearningHybrid2024,
  title = {Automatically Learning Hybrid Digital Twins of Dynamical Systems},
  booktitle = {Advances in {{Neural Information Processing Systems}} 37},
  author = {Holt, Samuel and Liu, Tennison and van der Schaar, Mihaela},
  year = 2024,
  pages = {72170--72218}
}

@article{honkelaVariationalLearningBitsback2004,
  title = {Variational Learning and Bits-Back Coding: {{An}} Information-Theoretic View to {{Bayesian}} Learning},
  shorttitle = {Variational Learning and Bits-Back Coding},
  author = {Honkela, Antti and Valpola, Harri},
  year = 2004,
  month = jul,
  journal = {IEEE Transactions on Neural Networks},
  volume = {15},
  number = {4},
  pages = {800--810},
  issn = {1941-0093},
  doi = {10.1109/TNN.2004.828762},
  urldate = {2026-04-26}
}

@inproceedings{izmailovAveragingWeightsLeads2018,
  title = {Averaging Weights Leads to Wider Optima and Better Generalization},
  booktitle = {Proceedings of the 34th {{Conference}} on {{Uncertainty}} in {{Artificial Intelligence}}},
  author = {Izmailov, Pavel and Podoprikhin, Dmitrii and Garipov, Timur and Vetrov, Dmitry and Wilson, Andrew Gordon},
  year = 2018
}

@inproceedings{kashtanovaDeepLearningModel2022,
  title = {Deep Learning for Model Correction in Cardiac Electrophysiological Imaging},
  booktitle = {Proceedings of the 5th {{International Conference}} on {{Medical Imaging}} with {{Deep Learning}}},
  author = {Kashtanova, Victoriya and Ayed, Ibrahim and Arrieula, Andony and Potse, Mark and Gallinari, Patrick and Sermesant, Maxime},
  year = 2022,
  pages = {665--675}
}

@inproceedings{keskarLargebatchTrainingDeep2017,
  title = {On Large-Batch Training for Deep Learning: {{Generalization}} Gap and Sharp Minima},
  booktitle = {Proceedings of the 5th {{International Conference}} on {{Learning Representations}}},
  author = {Keskar, Nitish Shirish and Mudigere, Dheevatsa and Nocedal, Jorge and Smelyanskiy, Mikhail and Tang, Ping Tak Peter},
  year = 2017
}

@inproceedings{kimFisherSAMInformation2022,
  title = {Fisher {{SAM}}: Information Geometry and Sharpness Aware Minimisation},
  booktitle = {Proceedings of the 39th {{International Conference}} on {{Machine Learning}}},
  author = {Kim, Minyoung and Li, Da and Hu, Shell X. and Hospedales, Timothy},
  year = 2022,
  pages = {11148--11161}
}

@inproceedings{kwonASAMAdaptiveSharpnessaware2021,
  title = {{{ASAM}}: {{Adaptive}} Sharpness-Aware Minimization for Scale-Invariant Learning of Deep Neural Networks},
  booktitle = {Proceedings of the 38th {{International Conference}} on {{Machine Learning}}},
  author = {Kwon, Jungmin and Kim, Jeongseop and Park, Hyunseo and Choi, In Kwon},
  year = 2021,
  pages = {5905--5914}
}

@inproceedings{mehtaNeuralDynamicalSystems2021,
  title = {Neural Dynamical Systems: {{Balancing}} Structure and Flexibility in Physical Prediction},
  booktitle = {Proceedings of the 2021 {{IEEE Conference}} on {{Decision}} and {{Control}}},
  author = {Mehta, Viraj and Char, Ian and Neiswanger, Willie and Chung, Youngseog and Nelson, Andrew Oakleigh and Boyer, Mark D. and Kolemen, Egemen and Schneider, Jeff},
  year = 2021,
  pages = {3735--3742},
  urldate = {2022-11-04}
}

@inproceedings{millerLearningInsulinglucoseDynamics2020,
  title = {Learning Insulin-Glucose Dynamics in the Wild},
  booktitle = {Proceedings of the 5th {{Machine Learning}} for {{Healthcare Conference}}},
  author = {Miller, Andrew C. and Foti, Nicholas J. and Fox, Emily},
  year = 2020,
  pages = {172--197}
}

@inproceedings{mollenhoffSAMOptimalRelaxation2022,
  title = {{{SAM}} as an Optimal Relaxation of {{Bayes}}},
  booktitle = {Proceedings of the 11th {{International Conference}} on {{Learning Representations}}},
  author = {M{\"o}llenhoff, Thomas and Khan, Mohammad Emtiyaz},
  year = 2022
}

@misc{palumboHybridModelingPhotoplethysmography2025,
  title = {Hybrid Modeling of Photoplethysmography for Non-Invasive Monitoring of Cardiovascular Parameters},
  author = {Palumbo, Emanuele and Saengkyongam, Sorawit and Cervera, Maria R. and Behrmann, Jens and Miller, Andrew C. and Sapiro, Guillermo and {Heinze-Deml}, Christina and Wehenkel, Antoine},
  year = 2025,
  number = {arXiv:2511.14452},
  eprint = {2511.14452},
  publisher = {arXiv},
  urldate = {2026-01-27},
  archiveprefix = {arXiv}
}

@article{psichogiosHybridNeuralNetworkfirst1992,
  title = {A Hybrid Neural Network-First Principles Approach to Process Modeling},
  author = {Psichogios, Dimitris C. and Ungar, Lyle H.},
  year = 1992,
  journal = {AIChE Journal},
  volume = {38},
  number = {10},
  pages = {1499--1511}
}

@inproceedings{qianIntegratingExpertODEs2021,
  title = {Integrating Expert {{ODEs}} into Neural {{ODEs}}: {{Pharmacology}} and Disease Progression},
  booktitle = {Advances in {{Neural Information Processing Systems}} 34},
  author = {Qian, Zhaozhi and Zame, William R. and Fleuren, Lucas M. and Elbers, Paul and {van der Schaar}, Mihaela},
  year = 2021,
  pages = {11364--11383}
}

@inproceedings{rico-martinezContinuoustimeNonlinearSignal1994,
  title = {Continuous-Time Nonlinear Signal Processing: {{A}} Neural Network Based Approach for Gray Box Identification},
  booktitle = {Proceedings of the {{IEEE Workshop}} on {{Neural Networks}} for {{Signal Processing}}},
  author = {{Rico-Mart{\'i}nez}, R. and Anderson, J. S. and Kevrekidis, I. G.},
  year = 1994,
  pages = {596--605}
}

@article{rissanenModelingShortestData1978,
  title = {Modeling by Shortest Data Description},
  author = {Rissanen, J.},
  year = 1978,
  journal = {Automatica},
  volume = {14},
  number = {5},
  pages = {465--471}
}

@article{rudolphHybridModelingDesign2024,
  title = {Hybrid Modeling Design Patterns},
  author = {Rudolph, Maja and Kurz, Stefan and Rakitsch, Barbara},
  year = 2024,
  journal = {Journal of Mathematics in Industry},
  volume = {14},
  number = {1},
  pages = {3}
}

@article{salzmannRealtimeNeuralMPC2023,
  title = {Real-Time Neural {{MPC}}: {{Deep}} Learning Model Predictive Control for Quadrotors and Agile Robotic Platforms},
  author = {Salzmann, Tim and Kaufmann, Elia and Arrizabalaga, Jon and Pavone, Marco and Scaramuzza, Davide and Ryll, Markus},
  year = 2023,
  journal = {IEEE Robotics and Automation Letters},
  volume = {8},
  number = {4},
  pages = {2397--2404}
}

@article{schweidtmannReviewPerspectiveHybrid2024,
  title = {A Review and Perspective on Hybrid Modeling Methodologies},
  author = {Schweidtmann, Artur M. and Zhang, Dongda and Von Stosch, Moritz},
  year = 2024,
  journal = {Digital Chemical Engineering},
  volume = {10},
  pages = {100136}
}

@inproceedings{senoufInductiveDomainTransfer2025,
  title = {Inductive Domain Transfer in Misspecified Simulation-Based Inference},
  booktitle = {Neural {{Information Processing Systems}} 38},
  author = {Senouf, Ortal and Wehenkel, Antoine and {Vincent-Cuaz}, C{\'e}dric and Abb{\'e}, Emmanuel and Frossard, Pascal},
  year = 2025
}

@inproceedings{shirakamiQTNetTheorybasedQueue2023,
  title = {{{QTNet}}: {{Theory-based}} Queue Length Prediction for Urban Traffic},
  booktitle = {Proceedings of the 29th {{ACM SIGKDD Conference}} on {{Knowledge Discovery}} and {{Data Mining}}},
  author = {Shirakami, Ryu and Kitahara, Toshiya and Takeuchi, Koh and Kashima, Hisashi},
  year = 2023,
  pages = {4832--4841}
}

@inproceedings{singhVariationalGreyboxDynamics2026,
  title = {Variational Grey-Box Dynamics Matching},
  booktitle = {Proceedings of the 29th {{International Conference}} on {{Artificial Intelligence}} and {{Statistics}}},
  author = {Singh, Gurjeet Sangra and Lavda, Frantzeska and Mercatali, Giangiacomo and Kalousis, Alexandros},
  year = {2026}
}

@inproceedings{takeishiDeepGreyboxModeling2023,
  title = {Deep Grey-Box Modeling with Adaptive Data-Driven Models toward Trustworthy Estimation of Theory-Driven Models},
  booktitle = {Proceedings of the 26th {{International Conference}} on {{Artificial Intelligence}} and {{Statistics}}},
  author = {Takeishi, Naoya and Kalousis, Alexandros},
  year = 2023,
  pages = {4089--4100},
  copyright = {All rights reserved}
}

@inproceedings{takeishiPhysicsintegratedVariationalAutoencoders2021,
  title = {Physics-Integrated Variational Autoencoders for Robust and Interpretable Generative Modeling},
  booktitle = {Advances in {{Neural Information Processing Systems}} 34},
  author = {Takeishi, Naoya and Kalousis, Alexandros},
  year = 2021,
  pages = {14809--14821}
}

@article{thompsonModelingChemicalProcesses1994,
  title = {Modeling Chemical Processes Using Prior Knowledge and Neural Networks},
  author = {Thompson, Michael L. and Kramer, Mark A.},
  year = 1994,
  journal = {AIChE Journal},
  volume = {40},
  number = {8},
  pages = {1328--1340}
}

@article{thoreauPhysicsinformedVariationalAutoencoders2025,
  title = {Physics-Informed Variational Autoencoders for Improved Robustness to Environmental Factors of Variation},
  author = {Thoreau, Romain and Risser, Laurent and Achard, V{\'e}ronique and Berthelot, B{\'e}atrice and Briottet, Xavier},
  year = 2025,
  journal = {Machine Learning},
  volume = {114},
  pages = {198}
}

@inproceedings{tsuzukuNormalizedFlatMinima2020,
  title = {Normalized Flat Minima: {{Exploring}} Scale Invariant Definition of Flat Minima for Neural Networks Using {{PAC-bayesian}} Analysis},
  booktitle = {Proceedings of the 37th {{International Conference}} on {{Machine Learning}}},
  author = {Tsuzuku, Yusuke and Sato, Issei and Sugiyama, Masashi},
  year = 2020,
  pages = {9636--9647}
}

@article{tusharDeepPhysicsCorrector2023,
  title = {Deep Physics Corrector: {{A}} Physics Enhanced Deep Learning Architecture for Solving Stochastic Differential Equations},
  author = {{Tushar} and Chakraborty, Souvik},
  year = 2023,
  journal = {Journal of Computational Physics},
  volume = {479},
  pages = {112004}
}

@inproceedings{vermaClimODEClimateWeather2024,
  title = {{{ClimODE}}: {{Climate}} and Weather Forecasting with Physics-Informed Neural {{ODEs}}},
  booktitle = {Proceedings of the 12th {{International Conference}} on {{Learning Representations}}},
  author = {Verma, Yogesh and Heinonen, Markus and Garg, Vikas},
  year = 2024
}

@inproceedings{wehenkelAddressingMisspecificationSimulationbased2025,
  title = {Addressing Misspecification in Simulation-Based Inference through Data-Driven Calibration},
  booktitle = {Proceedings of the 42nd {{International Conference}} on {{Machine Learning}}},
  author = {Wehenkel, Antoine and Gamella, Juan L. and Sener, Ozan and Behrmann, Jens and Sapiro, Guillermo and Cuturi, Marco and Jacobsen, J{\"o}rn-Henrik},
  year = 2025,
  pages = {65949--65980}
}

@article{wehenkelRobustHybridLearning2023,
  title = {Robust Hybrid Learning with Expert Augmentation},
  author = {Wehenkel, Antoine and Behrmann, Jens and Hsu, Hsiang and Sapiro, Guillermo and Louppe, Gilles and Jacobsen, J{\"o}rn-Henrik},
  year = 2023,
  journal = {Transactions of Machine Learning Research}
}

@article{wielandStructuralPracticalIdentifiability2021,
  title = {On Structural and Practical Identifiability},
  author = {Wieland, Franz-Georg and Hauber, Adrian L. and Rosenblatt, Marcus and T{\"o}nsing, Christian and Timmer, Jens},
  year = 2021,
  journal = {Current Opinion in Systems Biology},
  volume = {25},
  pages = {60--69}
}

@inproceedings{xuGeneralizingWeatherForecast2024,
  title = {Generalizing Weather Forecast to Fine-Grained Temporal Scales via Physics-{{AI}} Hybrid Modeling},
  booktitle = {Advances in {{Neural Information Processing Systems}}},
  author = {Xu, Wanghan and Ling, Fenghua and Han, Tao and Chen, Hao and Ouyang, Wanli and Bai, Lei},
  year = 2024,
  volume = {37},
  pages = {23325--23351}
}

@inproceedings{yeIdentifiabilityHybridDeep2024,
  title = {On the Identifiability of Hybrid Deep Generative Models: {{Meta-learning}} as a Solution},
  booktitle = {Advances in {{Neural Information Processing Systems}} 37},
  author = {Ye, Yubo and Tolou, Maryam and Vadhavkar, Sumeet and Jiang, Xiajun and Liu, Huafeng and Wang, Linwei},
  year = 2024,
  pages = {7714--7735}
}

@inproceedings{yinAugmentingPhysicalModels2021,
  title = {Augmenting Physical Models with Deep Networks for Complex Dynamics Forecasting},
  booktitle = {Proceedings of the 9th {{International Conference}} on {{Learning Representations}}},
  author = {Yin, Yuan and Le Guen, Vincent and Dona, J{\'e}r{\'e}mie and {de B{\'e}zenac}, Emmanuel and Ayed, Ibrahim and Thome, Nicolas and Gallinari, Patrick},
  year = 2021
}

@inproceedings{zhuangSurrogateGapMinimization2022,
  title = {Surrogate Gap Minimization Improves Sharpness-Aware Training},
  booktitle = {Proceedings of the 10th {{International Conference}} on {{Learning Representations}}},
  author = {Zhuang, Juntang and Gong, Boqing and Yuan, Liangzhe and Cui, Yin and Adam, Hartwig and Dvornek, Nicha and Tatikonda, Sekhar and Duncan, James and Liu, Ting},
  year = 2022
}

@inproceedings{zouHybrid$^2$NeuralODE2024,
  title = {Hybrid{\textsuperscript{2}} Neural {{ODE}} Causal Modeling and an Application to Glycemic Response},
  booktitle = {Proceedings of the 41st {{International Conference}} on {{Machine Learning}}},
  author = {Zou, Bob Junyi and Levine, Matthew E. and Zaharieva, Dessi P. and Johari, Ramesh and Fox, Emily},
  year = 2024,
  pages = {62934--62963}
}
\bibliographystyle{unsrtnat}

%%%%%%%%%%%%%%%%%%%%%%%%%%%%%%%%%%%%%%%%%%%%%%%%%%%%%%%%%%%%

\appendix

\numberwithin{theorem}{section}
\numberwithin{assumption}{section}
\numberwithin{proposition}{section}
\numberwithin{corollary}{section}
\numberwithin{definition}{section}

\section{Sharpness-Aware Minimization}
\label{appendix:sam}

Sharpness-aware minimization (SAM) \citep{foretSharpnessawareMinimizationEfficiently2021} is a method that aims to learn deep neural networks with the flatness (or sharpness) of the loss taken into account.
SAM and its variants \citep[e.g.,][]{kwonASAMAdaptiveSharpnessaware2021,kimFisherSAMInformation2022,zhuangSurrogateGapMinimization2022,mollenhoffSAMOptimalRelaxation2022} have been shown to improve the generalization capability of deep neural nets in typical tasks such as image classification.

The idea of SAM is to minimize the worst-case loss within a certain neighborhood around the current parameters.
Let $L_S(w)$ be a loss function on a training dataset $S$ for parameters $w$ of a model.
SAM tries to solve the following minimax problem:
\begin{equation*}
    \min_w \ \max_{\Vert \epsilon \Vert \leq \rho} \ L_S(w + \epsilon),
\end{equation*}
where $\rho > 0$ is a hyperparameter that defines the size of the neighborhood.
The inner maximization can be approximated by a single step of gradient ascent, yielding the perturbation:
\begin{equation}\label{eq:eps_sam}
    \epsilon^*_\text{SAM} = \rho \frac{\nabla L_S(w)}{\Vert \nabla L_S(w) \Vert}.
\end{equation}
The update rule of SAM is then given as
\begin{equation*}
    w \leftarrow w - \eta \nabla L_S(w + \epsilon^*_\text{SAM}),
\end{equation*}
where $\eta > 0$ is the learning rate.
The update step is performed by 1) computing the perturbation at the current parameter, 2) evaluating the loss gradient at the perturbed point, and 3) updating the parameter with the computed gradient.
In practice the higher-order derivative through $\epsilon^*$ is usually ignored \citep{foretSharpnessawareMinimizationEfficiently2021}.

To compute and minimize the sharpness in a meaningful manner, the geometry of the parameter manifold should be taken into account.
It has been addressed by multiple researchers based on different ideas.
For example, \citet{kwonASAMAdaptiveSharpnessaware2021} proposed adaptive SAM to scale the perturbation by the magnitude of each parameter, i.e.,
\begin{equation}\label{eq:eps_asam}
    \epsilon^*_\text{ASAM} = \rho \frac{w^2 \nabla L_S(w)}{\Vert w \nabla L_S(w) \Vert} \quad\text{(elementwise)}.
\end{equation}
\citet{kimFisherSAMInformation2022} proposed to use the loss curvature information for scaling.
Let $F(w)$ be the expected Hessian matrix of the loss (i.e., the Fisher information matrix) at $w$.
Then, Fisher SAM defines the perturbation as
\begin{equation}\label{eq:eps_fsam}
    \epsilon^*_\text{FSAM} = \rho \frac{F(w)^{-1} \nabla L_S(w)}{\sqrt{\nabla L_S(w)^\top F(w)^{-1} \nabla L_S(w)}},
\end{equation}
meaning that the perturbation is scaled by the local curvature of the loss landscape.

%%%%%%%%%%%%%%%%%%%%%%%%%%%%%%%%%%%%%%%%%%%%%%%%%%%%%%%%%%%%%%%%%%%%%

\section{Loss Bound}
\label{appendix:theory}

The purpose of the analysis here is to show that the empirical SAM objective,
\begin{equation*}
    L_S^\text{sam}(\theta, \phi) = \max_{\|\epsilon\| \leq \rho} L_S(\theta, \phi + \epsilon),
\end{equation*}
where $L_S(\theta, \phi) = \sum_{(x,y) \in S} \ell(h_{\theta,\phi}(x), y)$, and $S$ is an i.i.d. training dataset of size $n$, gives an upper bound on the expected population loss,
\begin{equation*}
    \mathbb{E}_{\phi \sim q} L(\theta, \phi).
\end{equation*}
For the sake of discussion, without losing much practicality, we restrict the posterior $q$ to be a Gaussian distribution
\begin{equation*}
    q = q_\varphi(\phi) \coloneqq \mathcal{N}(\phi \mid \varphi, \tau^2 I),
\end{equation*}
with some mean $\varphi \in \Phi$ and variance $\tau^2 \geq 0$.

In upper bounding $\mathbb{E}_{\phi \sim q_\varphi} L(\theta, \phi)$, we need to take care of the two arguments, $\theta$ and $\varphi$.
We follow the PAC-Bayes argument given in the SAM paper by \citet{foretSharpnessawareMinimizationEfficiently2021} for bounding the $\varphi$-part and then expand the discussion for all $\theta$ using the uniform-convergence argument based on covering numbers.

\begin{assumption}\label{assumption:1}
    The value of $\ell$ is bounded in $[0,1]$.
\end{assumption}
\begin{assumption}\label{assumption:2}
    For every $\phi \in \Phi$ and $(x,y)$, the map $\theta \mapsto \ell(h_{\theta,\phi}(x), y)$ is $\lambda$-Lipschitz on $\Theta \subset \mathbb{R}^{d_\theta}$ in the $\ell_2$ norm.
\end{assumption}
\begin{assumption}\label{assumption:3}
    $\Theta$ is a bounded set.
\end{assumption}
\begin{theorem}
    Suppose Assumptions~\ref{assumption:1}, \ref{assumption:2}, and \ref{assumption:3} hold.
    Let $\tau = \frac{\rho}{\sqrt{d_\varphi} + s}$ with some $s \geq 0$.
    Let $N(r, \Theta)$ with some $r > 0$ denote the $r$-covering number of $\Theta$ in the $\ell_2$ norm.
    Then, with probability at least $1 - \delta$ over the choice of $S$,
    \begin{multline}
        \mathbb{E}_{\phi \sim q_\varphi} L(\theta, \phi)
        \leq
        L_S^\text{sam}(\theta, \varphi)
        + 2 \lambda r
        + \exp \left( \frac{-(\frac{\rho}{\tau} - \sqrt{d_\varphi})^2}2 \right)
        \\
        + \sqrt{\frac{\frac14 d_\varphi \log \left( 1 + \frac{\|\varphi\|_2^2}{d_\varphi \tau^2} \right) + \frac14 + \log\frac{n N(r, \Theta)}{\delta} + 2\log(6n+3 d_\varphi)}{n-1}},
    \end{multline}
    for all $\theta \in \Theta$ and $\varphi \in \Phi$.
\end{theorem}
\begin{proof}
    Let $\Theta_r \subset \Theta$ be a minimal $r$-cover of $\Theta$ in the $\ell_2$ norm; by definition, $|\Theta_r| = N(r, \Theta)$.

    % Step 1: PAC-Bayes bound on a cover
    From the definition of $q_\varphi$, the target quantity can be written in the reparametrized form:
    \begin{equation}\label{eq:reparam}
        \mathbb{E}_{\phi \sim q_\varphi} L(\theta, \phi)
        = \mathbb{E}_{\epsilon \sim \mu_\tau} L(\theta, \varphi + \epsilon),
    \end{equation}
    where $\mu_\tau \coloneqq \mathcal{N}(0, \tau^2 I)$.
    For each fixed $\theta_k \in \Theta_r$, the discussion in \citet{foretSharpnessawareMinimizationEfficiently2021} gives the following PAC-Bayes bound: with probability at least $1 - \delta'$ over the choice of $S$, for all $\varphi \in \Phi$,
    \begin{multline}\label{eq:foret_at_cover}
        \mathbb{E}_{\epsilon \sim \mu_\tau} L(\theta_k, \varphi + \epsilon)
        \leq \mathbb{E}_{\epsilon \sim \mu_\tau} L_S(\theta_k, \varphi + \epsilon)
        \\
        + \sqrt{\frac{\frac14 d_\varphi \log \left( 1 + \frac{\|\varphi\|_2^2}{d_\varphi \tau^2} \right) + \frac14 + \log\frac{n}{\delta'} + 2\log(6n + 3 d_\varphi)}{n-1}}.
    \end{multline}
    Note that this bound holds uniformly in $\varphi$ on a single $1-\delta'$ event over $S$, since the PAC-Bayes inequality is uniform in the posterior mean.
    Applying the union bound over the $N(r, \Theta)$ points of $\Theta_r$ and setting $\delta' = \delta / N(r, \Theta)$, we obtain that with probability at least $1 - \delta$ over $S$, \eqref{eq:foret_at_cover} holds simultaneously for all $\theta_k \in \Theta_r$ and $\varphi \in \Phi$, with $\log\frac{n}{\delta'}$ replaced by $\log\frac{n N(r, \Theta)}{\delta}$.

    % Step 2: Lipschitz lifting from the cover to all of $\Theta$
    For any $\theta \in \Theta$, by definition of the covering, there exists $\theta_k \in \Theta_r$ with $\|\theta - \theta_k\|_2 \leq r$.
    By assumption, $\theta \mapsto \ell(h_{\theta, \phi}(x), y)$ is $\lambda$-Lipschitz uniformly in $\phi$, $x$, and $y$.
    Since $L_S$ is the empirical average and $L$ is the population expectation of $\ell$, both inherit $\lambda$-Lipschitzness in $\theta$ for every $\phi'$:
    \begin{equation*}
        |L(\theta, \phi') - L(\theta_k, \phi')| \leq \lambda r,
        \quad
        |L_S(\theta, \phi') - L_S(\theta_k, \phi')| \leq \lambda r.
    \end{equation*}
    Taking expectation over $\phi' \sim q_\varphi$ (equivalently, $\phi' = \varphi + \epsilon$ with $\epsilon \sim \mu_\tau$) preserves the inequalities, so
    \begin{align*}
        \mathbb{E}_{\epsilon \sim \mu_\tau} L(\theta, \varphi + \epsilon)
        &\leq \mathbb{E}_{\epsilon \sim \mu_\tau} L(\theta_k, \varphi + \epsilon) + \lambda r,
        \\
        \mathbb{E}_{\epsilon \sim \mu_\tau} L_S(\theta_k, \varphi + \epsilon)
        &\leq \mathbb{E}_{\epsilon \sim \mu_\tau} L_S(\theta, \varphi + \epsilon) + \lambda r.
    \end{align*}
    Chaining these with \eqref{eq:foret_at_cover} after the union bound, with probability at least $1 - \delta$, for all $\theta \in \Theta$ and $\varphi \in \Phi$,
    \begin{multline}\label{eq:lifted}
        \mathbb{E}_{\epsilon \sim \mu_\tau} L(\theta, \varphi + \epsilon)
        \leq \mathbb{E}_{\epsilon \sim \mu_\tau} L_S(\theta, \varphi + \epsilon)
        + 2 \lambda r
        \\
        + \sqrt{\frac{\frac14 d_\varphi \log \left( 1 + \frac{\|\varphi\|_2^2}{d_\varphi \tau^2} \right) + \frac14 + \log\frac{n N(r, \Theta)}{\delta} + 2\log(6n + 3 d_\varphi)}{n-1}}.
    \end{multline}

    % Step 3: Replace Gaussian-perturbed L_S with L_S^{sam}
    Splitting the expectation to the inside and outside of a ball of radius $\rho$, the first term of the rhs of \eqref{eq:lifted} is
    \begin{equation*}
        \mathbb{E}_{\epsilon \sim \mu_\tau} L_S(\theta, \varphi + \epsilon)
        = \int_{\|\epsilon\| \leq \rho} L_S(\theta, \varphi+\epsilon) d\mu_\tau(\epsilon)
        + \int_{\|\epsilon\| > \rho} L_S(\theta, \varphi+\epsilon) d\mu_\tau(\epsilon).
    \end{equation*}
    Since the loss is bounded in $[0,1]$,
    \begin{equation}\label{eq:expectation_split}
        \mathbb{E}_{\epsilon \sim \mu_\tau} L_S(\theta, \varphi + \epsilon)
        \leq L_S^\text{sam}(\theta, \varphi) \mu_\tau(B_\rho) + \mu_\tau(B_\rho^\mathrm{c})
        \leq L_S^\text{sam}(\theta, \varphi) + \mu_\tau(B_\rho^\mathrm{c}),
    \end{equation}
    where $B_\rho$ is the ball of radius $\rho$ around $0$, and $B_\rho^\mathrm{c}$ denotes the complement of $B_\rho$.

    % Step 4: Gaussian-norm tail
    For $\epsilon \sim \mu_\tau$, write $\epsilon = \tau z$ with $z \sim \mathcal{N}(0, I)$.
    The map $z \mapsto \|z\|_2$ is $1$-Lipschitz, and Jensen's inequality tells $\mathbb{E}\|z\| \leq \sqrt{\mathbb{E} \|z\|^2} = \sqrt{d_\varphi}$.
    The Borell--TIS inequality gives, for every $t \geq 0$,
    \begin{equation*}
        \operatorname{Pr} \big( \|z\| \geq \mathbb{E}\|z\| + t \big) \leq e^{-\frac{t^2}2},
    \end{equation*}
    and consequently,
    \begin{equation*}
        \operatorname{Pr} \big( \|\epsilon\| \geq \tau(\sqrt{d_\varphi} + t) \big) \leq e^{-\frac{t^2}2}.
    \end{equation*}
    Setting $t = \frac{\rho}{\tau} - \sqrt{d_\varphi} = s \geq 0$,
    \begin{equation}\label{eq:gaussian_bound}
        \mu_\tau(B_\rho^\mathrm{c}) = \operatorname{Pr} \big( \|\epsilon\| > \rho \big)
        \leq e^{-\frac{(\frac{\rho}{\tau} - \sqrt{d_\varphi})^2}2}.
    \end{equation}

    % Step 5: Assembly
    Substituting \eqref{eq:expectation_split} and \eqref{eq:gaussian_bound} into \eqref{eq:lifted}, and using the reparametrization \eqref{eq:reparam}, with probability at least $1 - \delta$, for all $\theta \in \Theta$ and $\varphi \in \Phi$,
    \begin{multline*}
        \mathbb{E}_{\phi \sim q_\varphi} L(\theta, \phi)
        \leq L_S^\text{sam}(\theta, \varphi)
        + 2 \lambda r
        + \exp \left( \frac{-(\frac{\rho}{\tau} - \sqrt{d_\varphi})^2}2 \right)
        \\
        + \sqrt{\frac{\frac14 d_\varphi \log \left( 1 + \frac{\|\varphi\|_2^2}{d_\varphi \tau^2} \right) + \frac14 + \log\frac{n N(r, \Theta)}{\delta} + 2\log(6n + 3 d_\varphi)}{n-1}}.
    \end{multline*}
\end{proof}

%%%%%%%%%%%%%%%%%%%%%%%%%%%%%%%%%%%%%%%%%%%%%%%%%%%%%%%%%%%%%%%%%%%%%

\section{Identifiability Condition}
\label{appendix:identifiability}

We try to formulate a condition for the identification of $\theta$ more formally than in the main text and show that such a condition does not depend on specific model architectures.
The target of the analysis here is the free energy in \cref{eq:soft_profile_loss}:
\begin{equation*}
    \bar{J}(\theta) \coloneqq \min_q \Big( \mathbb{E}_{\phi \sim q} L(\theta, \phi) + \operatorname{KL}(q(\phi) \ \| \ p(\phi)) \Big).
\end{equation*}
We aim to state the conditions under which this $\bar{J}$ is helpful for identifying the scientific parameter, $\theta$.

Let $\theta_0 \in \Theta$ be the ``correct'' value of the scientific parameter.
What a ``correct'' scientific parameter means depends on context, and we do not limit the scope of discussion here in this regard.
Let $\phi_0 \in \Phi$ be the corresponding value of the machine learning parameter, that is, $\phi_0 \in \arg\min_\phi L(\theta_0, \phi)$.
We assume that $\theta_0$ is identifiable if we fix the machine learning parameter at $\phi_0$.
More formally:
\begin{assumption}\label{assumption:identifiability}
    Let $\delta_{\phi_0}(\theta) \coloneqq L(\theta, \phi_0) - L(\theta_0, \phi_0)$.
    Under the fixed $\phi_0$, $\theta_0$ is identifiable in the sense that $\delta_{\phi_0}(\theta) > 0$ for all $\theta \neq \theta_0$.
\end{assumption}

The core of our motivation in the hybrid modeling study is the observation that, regardless of the original identifiability assumption under the fixed, oracle $\phi_0$, the profile loss:
\begin{equation*}
    J(\theta) = \min_\phi L(\theta, \phi)
\end{equation*}
loses the identifiability, at least practically, because when the machine learning part of a hybrid model is highly flexible, the $\min_\phi$ operation yields $J(\theta) \approx J(\theta_0)$ for many (or even all) $\theta \neq \theta_0$.

The previous studies on hybrid model learning \citep{yinAugmentingPhysicalModels2021,takeishiPhysicsintegratedVariationalAutoencoders2021,takeishiDeepGreyboxModeling2023} used and analyzed some regularizers to break the ties in the profile loss.
They are certainly helpful when the model architecture follows the design idea of the regularizers.
However, the need to design regularizers specifically for each architecture hinders the broad application of hybrid modeling with various architecture choices.

Now let us see under what conditions $\bar{J}$ may help the identification.
It is well known that the minimizer of the $\min_q$ operation in the definition of $\bar{J}$ is
\begin{equation*}
    q^*_\theta(\phi) \coloneqq \frac{p(\phi) \exp(-L(\theta, \phi))}{Z(\theta)},
\end{equation*}
where the partition function is $Z(\theta) \coloneqq \int p(\phi) \exp(-L(\theta, \phi)) d \phi$.
A natural consequence is
\begin{equation*}
    \bar{J}(\theta) = -\log Z(\theta),
\end{equation*}
and thus
\begin{equation}\label{eq:V_representation}
    \bar{J}(\theta) = J(\theta) - \log V(\theta)
    \quad \text{where} \quad
    V(\theta) = \int \exp \Big( -( L(\theta,\phi) - J(\theta) ) \Big) p(\phi) d\phi.
\end{equation}
This view allows us to think that $\bar{J}(\theta)$ has the $-\log V(\theta)$ term as a regularizer.
From the definition of $V(\theta)$, it measures the mass of prior $p(\phi)$ corresponding to the region where $L(\theta,\phi)-J(\theta) \geq 0$ is small, that is, for a fixed $\theta$, the region of good $\phi$ in terms of the small excess loss.

Let
\begin{equation*}
    G_{\phi_0}(\theta) \coloneqq L(\theta, \phi_0) - J(\theta) \geq 0
\end{equation*}
be the loss improvement, under fixed $\theta$, given when $\phi$ is moved from $\phi_0$ to the profile minimizer, $\arg \min_\phi L(\theta,\phi)$.
This value reflects the extent of the compensation by the machine learning part for a wrong $\theta$.
We can formulate an identification condition as follows.
\begin{proposition}
    Let Assumption~\ref{assumption:identifiability} hold.
    If
    \begin{equation}\label{eq:identifiability_condition}
        - \log V(\theta) > - \log V(\theta_0) + G_{\phi_0}(\theta) - \delta_{\phi_0}(\theta) \quad\text{for}\quad \theta \neq \theta_0,
    \end{equation}
    then $\bar{J}$ can identify $\theta_0$ in the sense that $\bar{J}(\theta) - \bar{J}(\theta_0) > 0$ for $\theta \neq \theta_0$.
\end{proposition}
\begin{proof}
    Since $J(\theta) = L(\theta, \phi_0) - G_{\phi_0}(\theta)$ by the definition of $G_{\phi_0}$, we have
    \begin{equation*}
        J(\theta) - J(\theta_0)
        = L(\theta, \phi_0) - L(\theta_0, \phi_0) - G_{\phi_0}(\theta) + G_{\phi_0}(\theta_0)
        = \delta_{\phi_0}(\theta) - G_{\phi_0}(\theta),
    \end{equation*}
    because $G_{\phi_0}(\theta_0) = 0$ as we assumed $\phi_0 \in \arg \min_\phi L(\theta_0, \phi)$.
    Consequently, from \eqref{eq:V_representation},
    \begin{equation*}
        \bar{J}(\theta) - \bar{J}(\theta_0) = \delta_{\phi_0}(\theta) - G_{\phi_0}(\theta) + \log \frac{V(\theta_0)}{V(\theta)}.
    \end{equation*}
    Therefore, when \eqref{eq:identifiability_condition} holds, we have $\bar{J}(\theta) - \bar{J}(\theta_0) > 0$ for $\theta \neq \theta_0$.
\end{proof}

An important note is that we here never say that the use of $\bar{J}$ admits a free-hand identification of $\theta$.
On the contrary, the identification is ensured only when the condition \eqref{eq:identifiability_condition} is met, and we have not discussed the exact situations when it is met.
Rather, the point of the discussion above is that the identification condition \eqref{eq:identifiability_condition} does not presume specific model architectures.
Moreover, the global identification discussed above is probably overly strong, and arguments for weaker local identification would be useful for analyzing practical deep hybrid modeling.

%%%%%%%%%%%%%%%%%%%%%%%%%%%%%%%%%%%%%%%%%%%%%%%%%%%%%%%%%%%%%%%%%%%%%

% \section{Relation to Misspecified Model Estimation}
% \label{appendix:misspecified}

% \input{relation_to_misspcified}

%%%%%%%%%%%%%%%%%%%%%%%%%%%%%%%%%%%%%%%%%%%%%%%%%%%%%%%%%%%%%%%%%%%%%

\section{Task Description}
\label{appendix:tasks}

\subsection{Pendulum Time-Series}

\paragraph{Data}

We simulated the damped pendulum system: $$\ddot{v}(t) + \gamma \dot{v}(t) + (2 \pi \omega)^2 \sin v(t)=0,$$ where $v(t): \mathbb{R} \to \mathbb{R}$ is the pendulum's angle, and $\ddot{v}$ and $\dot{v}$ denote the second and first derivatives with regard to time $t \in \mathbb{R}$, respectively.
For the simulation we used the explicit Runge-Kutta method of order 5(4) for the numerical integration with a time step of $0.02$.
We then subsampled the sequence by a factor of 10, so the effective time step of the data is $\Delta t = 0.2$.
We set $\gamma=0.5$ and $\omega=2/3$ for data generation.

We formulate the task as the prediction from an initial state $x$ to a subsequent sequence of states $y$:
\begin{equation*}\begin{aligned}
    x &= \mathbf{v}(0) \in \mathbb{R}^{2} \quad \text{and} \\
    y &= [ \mathbf{v}(\Delta t), \dots, \mathbf{v}(m_y \Delta t) ] \in \mathbb{R}^{m_y \times 2},
\end{aligned}\end{equation*}
where $\mathbf{v}=[v, \dot{v}]$, and $m_y=10$ is the output sequence lengths.
The initial angle $v(0)$ and angular velocity $\dot{v}(0)$ were randomly sampled from the uniform distribution $U(-\pi/2, \pi/2)$.
We generated 25, 25, and 25 input-output pairs for the training, validation, and test data, respectively.
We added zero-mean Gaussian noise with standard deviation $0.01$ to the data.
Similar data were also used in \citet{yinAugmentingPhysicalModels2021}.

\paragraph{Model}

The hybrid model we used is a hybrid neural ODE as practiced in \citet{yinAugmentingPhysicalModels2021}.
We suppose that we know a part of the data-generating equation as an incomplete scientific model: $$f_\theta(v) \coloneqq \ddot{v}(t) + (2 \pi \tilde\omega)^2 \sin v(t),$$ where $\theta = \{ \tilde\omega \}$ is the unknown parameter to be identified.
Then the hybrid neural ODE is formulated as
\begin{equation*}
    y = \operatorname{ODESolve} ( f_\theta(v) + g_\phi(v) = 0; v(0) = x )
\end{equation*} where $g_\phi$ is a feed-forward neural net with two hidden layers of size 128 and the ReLU activation function.

%%%%%%%%%%%%%%%%%%%%%%%%%%%%%%%%%%%%%%%%%%%%%%%%%%%%%%%%%%%%%%%%%%%%%

\subsection{Reaction-Diffusion}

\paragraph{Data}

We simulated the FitzHugh-Nagumo type reaction-diffusion system: \begin{equation*}\begin{aligned}
    u_t(t,\xi) &= a \nabla^2 u + u - u^3 - \kappa - v, \\
    v_t(t,\xi) &= b \nabla^2 v + u - v,
\end{aligned}\end{equation*}
where $u(t,\xi), v(t,\xi): \mathbb{R} \times [-1,1]^2 \to \mathbb{R}$ are the concentrations of two chemical species at time $t \in \mathbb{R}$ and spatial location $\xi \in [-1,1]^2$.
$u_t$ and $v_t$ denote the temporal derivatives of $u$ and $v$, respectively.
$\nabla^2$ is the Laplacian operator with regard to the space.
For the simulation we used the explicit Runge-Kutta method of order 5(4) for the numerical integration with a time step of $0.001$.
We then subsampled the sequence by a factor of 100, so the effective time step of the data is $\Delta t = 0.1$.
We set $a=0.001$, $b=0.005$, and $\kappa=0.005$ for data generation.

The task is to predict from an initial state $x \in \mathbb{R}^{2 \times d \times d}$ to the subsequent sequence of states $y \in \mathbb{R}^{2 \times d \times d \times m_y}$, where $d=32$ is the number of spatial discretization points along each axis, and $m_y=5$ is the output sequence length.
We used the five-point stencil method for approximating the Laplacian operator.
The initial concentrations $u(0,\xi)$ and $v(0,\xi)$ were generated by sampling from the uniform distribution $U(0,1)$.
We generated 100, 100, and 100 input-output pairs for the training, validation, and test data, respectively.
We added zero-mean Gaussian noise with standard deviation $0.01$ to the data.
Similar data were also used in \citet{yinAugmentingPhysicalModels2021}.

%%%%% TODO: explicit x,y definition

\paragraph{Model}

We set an incomplete scientific model as the diffusion terms: $$f_\theta(u,v) \coloneqq \begin{bmatrix} u_t - \tilde a \nabla^2 u \\ v_t - \tilde b \nabla^2 v\end{bmatrix},$$ where $\theta = \{ \tilde a, \tilde b \}$ is the unknown parameters to be identified.
It is combined with a machine learning part $g_\phi$ implemented as a convolutional neural net (kernel size 3) with two hidden layers of size (i.e., number of channels) 16, the ReLU activation function, and batch normalization.
The combination is formulated as
\begin{equation*}
    y = \operatorname{ODESolve} ( f_\theta(u,v) + g_\phi(u,v) = 0; v(0) = x ).
\end{equation*}

%%%%%%%%%%%%%%%%%%%%%%%%%%%%%%%%%%%%%%%%%%%%%%%%%%%%%%%%%%%%%%%%%%%%%

\subsection{Duffing Oscillator}

\paragraph{Data}

We simulated a nonlinear oscillator system: $$\ddot{v}(t) + \alpha v(t) - \beta v(t)^3 = 0.$$
For the simulation we used the explicit Runge-Kutta method of order 5(4) for the numerical integration with a time step of $0.005$.
We then subsampled the sequence by a factor of 20, so the effective time step of the data is $\Delta t = 0.1$.
We set $\alpha=1$ and $\beta=1$ for data generation.

We formulate the task as the prediction from an initial condition $x$ to a subsequent states $y$ of length $m_y$:
\begin{equation*}\begin{aligned}
    x &= \mathbf{v}(0) \in \mathbb{R}^{2} \quad \text{and} \\
    y &= [ \mathbf{v}(\Delta t), \dots, \mathbf{v}(m_y \Delta t) ] \in \mathbb{R}^{m_y \times 2},
\end{aligned}\end{equation*}
where $\mathbf{v}=[v, \dot{v}]$, and $m_y=10$ is the output sequence lengths.
We generated 100, 100, and 100 input-output pairs for the training, validation, and test data, respectively.
We added zero-mean Gaussian noise with standard deviation $0.01$ to the data.

\paragraph{Model}

We set an incomplete scientific model: $$f_\theta(v) \coloneqq \ddot{v}(t) + \tilde\alpha v(t),$$ where $\theta = \{ \tilde\alpha \}$ is the unknown parameter to be identified.
The hybrid model is a hybrid neural ODE:
\begin{equation*}
    y = \operatorname{ODESolve} ( f_\theta(v) + g_\phi(v) = 0; v(0) = x )
\end{equation*} where $g_\phi$ is a feed-forward neural net with two hidden layers of size 128 and the ReLU activation function.

%%%%%%%%%%%%%%%%%%%%%%%%%%%%%%%%%%%%%%%%%%%%%%%%%%%%%%%%%%%%%%%%%%%%%

\subsection{Pendulum Images}

\paragraph{Data}

We simulated trajectories of a pendulum as in the \textsc{pendulum time-series} task and then rendered gray-scale images of size $d \times d$ with $d=48$ pixels based on the simulated pendulum angles.
The task is to predict future image sequences, $y \in [0,1]^{m_y \times d \times d}$, given a preceding sequence of images, $x \in [0,1]^{m_x \times d \times d}$, where $m_x=10$ and $m_y=10$ are the input and output sequence length, respectively.
We generated 200, 200, and 200 input-output pairs for the training, validation, and test data, respectively.

%%%%% TODO: example of pendulum images

\paragraph{Model}

The hybrid model is built by combining the hybrid neural ODE used in the \textsc{pendulum time-series} task and additional neural nets for encoding and decoding.
The model has two encoders.
The first encoder receives the input image sequence and gives the initial condition of the pendulum for the hybrid neural ODE.
The second encoder receives the input image sequence and gives a latent representation, which is then concatenated with the output of the hybrid neural ODE.
The decoder maps the output of the hybrid neural ODE (pendulum angles), concatenated with the latent representation, to images.
It comprises 1) a network with fully-connected layers with two hidden layers of size 1152 with ReLU and 2) a convolutional network (kernel size 3) with two hidden layers of size (i.e., number of channels) 576 and 288 with ReLU and batch normalization.

Note that the absolute initial condition of the pendulum is in principle unindentifiable because we do not give any correspondence between the images and the pendulum angles.
However it is not problematic here because we focus on the estimation of the pendulum's frequency parameter, $\theta = \{ \tilde\omega \}$, and the interest is not in the phase of the oscillation.
The encoder should learn the map from images to the initial condition only up to the translational invariance.

%%%%%%%%%%%%%%%%%%%%%%%%%%%%%%%%%%%%%%%%%%%%%%%%%%%%%%%%%%%%%%%%%%%%%

\subsection{Wind Tunnel}

\paragraph{Data}

We used a part of the dataset by \citet{gamellaCausalChambersRealworld2025}\footnote{\label{footnote:causalchamber}\url{https://github.com/juangamella/causal-chamber}}.
It is a dataset comprising measurements from a real-world wind tunnel.
We extracted a subset of the data (``loads\_hatch\_mix\_fast\_run'' of ``wt\_walks\_v1'' dataset) to formulate a task to predict a sequence of the pressure inside the tunnel $y \in \mathbb{R}^{m}$ from a sequence of the loads of the intake and exhaust fans and the position of the hatch, $x \in \mathbb{R}^{m \times 3}$, where $m$ is the length of the input and output sequences.
The hatch controls an additional opening at the middle of the tunnel.
We created the subsequences $\{(x,y)\}$ by a sliding window of length $m=200$ with a stride of 50 from the original time series.
Finally we have 600, 150, and 235 input-output pairs for the training, validation, and test data, respectively.

\paragraph{Model}

As an incomplete scientific model $f_\theta$, we use the models ``A1'' and ``C2'' presented in \citet{gamellaCausalChambersRealworld2025}, which give the relation between the pressure and the fan loads.
They are incomplete in terms of the effect of the hatch position and the transient dynamics of the pressure.
The hybrid model transforms the output of these models with a feed-forward neural net.
We set a scalar parameter of the model C2, namely the ratio $r \in [0,1]$ of the maximum airflow that the fans produce at full speed, as the unknown parameter to be estimated, i.e., $\theta = \{ r \}$.
As the truly data-generating value of the parameter is never known, we used a value suggested in the paper and the codes \citep{gamellaCausalChambersRealworld2025}, $r=0.7$, as a reference value to compute the estimation error of $\theta$.
The full hybrid model is $$y = g_\phi(z, x),$$ where $z = f_\theta(x) \in \mathbb{R}^{m}$ is the output of the scientific model, and $g_\phi$ is a feed-forward neural net with two hidden layers of size $2m$ and the ReLU activation function.

%%%%%%%%%%%%%%%%%%%%%%%%%%%%%%%%%%%%%%%%%%%%%%%%%%%%%%%%%%%%%%%%%%%%%

\subsection{Light Tunnel}

\paragraph{Data}

We used another part of the dataset by \citet{gamellaCausalChambersRealworld2025}\textsuperscript{\ref{footnote:causalchamber}}.
It contains real-world measurements from a light tunnel made of light sources, linear polarizers, and a light sensor.
We extracted a subset of the data (``uniform\_ap\_1.8\_iso\_500.0\_ss\_0.005'' of ``lt\_camera\_v1'' dataset) to formulate a task to predict the RGB image taken by the sensor, $y \in [0,1]^{3 \times 100 \times 100}$, from the light source configuration and the polarizer angles, $x = \begin{bmatrix} R & G & B & \alpha_1 & \alpha_2 \end{bmatrix} \in \mathbb{R}^{5}$.
$R,G,B$ are the intensities of the red, green, and blue light sources, respectively, and $\alpha_1, \alpha_2$ are the angles of the two polarizers.
Finally we have 6000, 1000, and 2500 input-output pairs for the training, validation, and test data, respectively.

\paragraph{Model}

As an incomplete scientific model, $f_\theta$, we use the model ``F3'' presented in \citet{gamellaCausalChambersRealworld2025}, which models the effect of the polarizers in a frequency-dependent way:
\begin{equation}\label{eq:model_f3}
    \begin{bmatrix} \tilde{R} \\ \tilde{G} \\ \tilde{B} \end{bmatrix}
    =
    \min \left( 1,
        e
        \operatorname{diag}(\bm{w})
        \bm{S}
        \operatorname{diag}\left( (\bm{t}^\mathrm{p} - \bm{t}^\mathrm{c}) \cos^2 (\alpha_1 - \alpha_2) + \bm{T}^c \right)
        \begin{bmatrix} R \\ G \\ B \end{bmatrix}.
    \right)
\end{equation}
We used the value of $e$ and $\bm{S}$ suggested in the paper and the codes \citep{gamellaCausalChambersRealworld2025}.
The remaining, $\bm{w} \in \mathbb{R}^3$, $\bm{t}^\mathrm{p} \in \mathbb{R}^3$, and $\bm{t}^\mathrm{c} \in \mathbb{R}^3$, are treated as unknown parameters to be estimated, i.e., $\theta \in \mathbb{R}^9$.
As evident from \cref{eq:model_f3}, the magnitude of these parameters is inherently unidentifiable.
We thus compared the estimation results and the reference values suggested in the paper and the codes \citep{gamellaCausalChambersRealworld2025} with the cosine similarity.
Given $\tilde{R}$, $\tilde{G}$, and $\tilde{B}$ computed by \cref{eq:model_f3}, then $f_\theta$ outputs an image by replicating the RGB values spatially in a hexagonal shape resembling the light placement pattern.
The full hybrid model further transforms such an image with a convolutional neural net with U-Net architecture.

%%%%%%%%%%%%%%%%%%%%%%%%%%%%%%%%%%%%%%%%%%%%%%%%%%%%%%%%%%%%%%%%%%%%%

\section{Learning and Model Selection}

For all tasks and learning strategies, we used the Adam optimizer for setting adaptive learning rates.
The overall learning rate of Adam was set to $10^{-4}$ for \textsc{pendulum time-series}, \textsc{reaction-diffusion}, and \textsc{duffing oscillator}; $3 \times 10^{-4}$ for \textsc{wind tunnel}; and $10^{-3}$ for \textsc{pendulum images} and \textsc{light tunnel}.
We run the optimization for 20000 iterations.
We did full-batch training for \textsc{pendulum time-series}, \textsc{reaction-diffusion}, \textsc{duffing oscillator}, and \textsc{wind tunnel}; and mini-batch training with a batch size of 50 for \textsc{pendulum images} and \textsc{light tunnel}.
We applied the cosine annealing schedule for the learning rate with a single period over the entire training iterations, i.e., the learning rate is gradually decreased from the initial value to $10^{-6}$.

We report the performance of a model that achieved the best $\theta$-error for all the methods.
Although it is against the normal practice to select models by validation prediction errors, we had to do so because the $y$-error and $\theta$-error are usually not consistent with finite training and validation datasets; it is, in the first place, the cause of our headache in hybrid modeling!
If the validation/test prediction error was a good proxy of the parameter estimation error, we (and other researchers) would not have done such a lot of research on this topic.
To maintain the comparison fair, we applied this model selection strategy not only to the proposed methods but also to all the baseline methods.

%%%%%%%%%%%%%%%%%%%%%%%%%%%%%%%%%%%%%%%%%%%%%%%%%%%%%%%%%%%%%%%%%%%%%

\section{More Results}
\label{appendix:results}

\subsection{Learning with Varying Scientific Parameters}

We addressed a more challenging problem setting, where the data are generated by varying values of $\theta$.
For the \textsc{pendulum time-series} task, we randomly sampled $\omega$ from the uniform distribution $U(0.1, 1.0)$ for each input-output pair.
The input $x$ is now a sequence $x = [ \mathbf{v}(- (m_x - 1) \Delta t), \dots, \mathbf{v}(0) ] \in \mathbb{R}^{m_x \times 2}$, where $m_x=10$ is the input sequence length, instead of the mere initial condition $\mathbf{v}(0)$ in the original task.
We added an encoder network to the model, which receives the input sequence and outputs an estimate of $\omega$.
The encoder network is a feed-forward neural net with two hidden layers of size 256 and the ReLU activation function.
$x$ is first flattened to a vector before feeding it to the encoder.

We altered the \textsc{reaction-diffusion} task similarly to the above.
We randomly sampled $a$ and $b$ from the uniform distributions $U(0, 0.02)$ and $U(0, 0.02)$, respectively, for each input-output pair.
The input $x$ is now a sequence of 2-channel 2-D fields of size $2 \times d \times d \times m_x$ with length $m_x=5$.
The encoder network comprises a convolutional network and a feed-forward network with fully-connected layers.
The convolutional net has two hidden layers of size 288 and 144, the ReLU activation function, and batch normalization.
We feed the input $x$ to the convolutional net by first flattening the temporal dimension into the channel dimension.
Its output is then fed to the feed-forward net, which has one hidden layer of size 256 and the ReLU activation function.

For the \textsc{duffing oscillator} task, we randomly sampled $\alpha$ from the uniform distribution $U(0, 2)$ for each input-output pair.
The input $x$ is now a sequence of length $m_x=10$.
The encoder network is a feed-forward neural net with two hidden layers of size 256 and the ReLU activation function.

In \cref{fig:varphys}, we show the outputs of the encoder networks.
Most notably, in the \textsc{duffing oscillator} task, the encoder learned with \texttt{fsam} infers $\theta$ with better correlation to the data-generating values, compared to the \texttt{erm} case.
Meanwhile there remain non-negligible variances in the inferred values, which suggests the difficulty of the problem.
On the other hand, the results of the \textsc{reaction-diffusion} task show almost no difference between the two methods, which implies that the task is relatively easy.

\begin{figure}[t]
    \centering
    \begin{minipage}[t]{0.48\textwidth}
        \centering
        \vspace*{0pt}
        \begin{minipage}[t]{0.48\linewidth}
            \centering
            \vspace*{0pt}
            \includegraphics[clip,width=\linewidth]{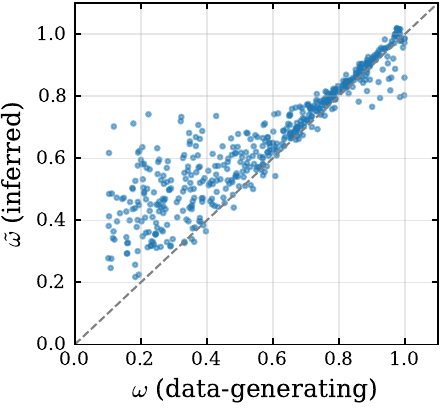}
            \\
            \hspace{1em}\texttt{erm} \footnotesize{($\text{RMSE}=0.098$)}
        \end{minipage}
        \hfill
        \begin{minipage}[t]{0.48\linewidth}
            \centering
            \vspace*{0pt}
            \includegraphics[clip,width=\linewidth]{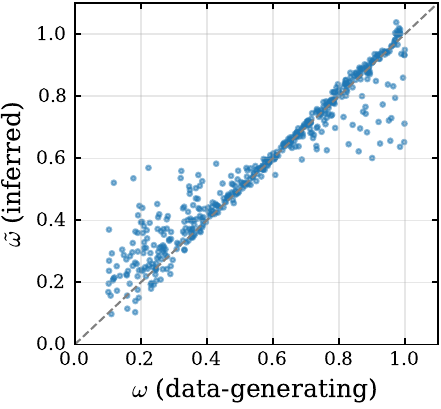}
            \\
            \hspace{1em}\texttt{fsam} \footnotesize{($\text{RMSE}=0.046$)}
        \end{minipage}
        \\[2ex]
        (a) \textsc{\textsc{pendulum time-series}} \footnotesize{($\theta = \{\omega\}$)}
    \end{minipage}
    %
    % \\[5ex]
    \hfill
    \begin{minipage}[t]{0.48\textwidth}
        \centering
        \vspace*{0pt}
        \begin{minipage}[t]{0.48\linewidth}
            \centering
            \vspace*{0pt}
            \includegraphics[clip,width=\linewidth]{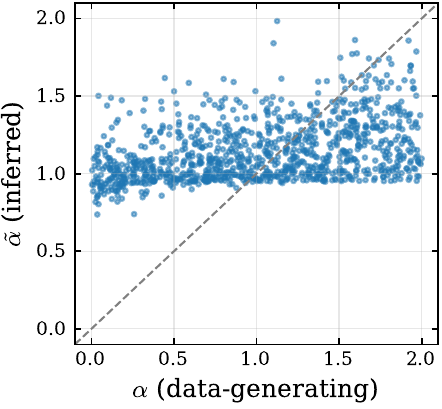}
            \\
            \hspace{1em}\texttt{erm} \footnotesize{($\text{RMSE}=0.45$)}
        \end{minipage}
        \hfill
        \begin{minipage}[t]{0.48\linewidth}
            \centering
            \vspace*{0pt}
            \includegraphics[clip,width=\linewidth]{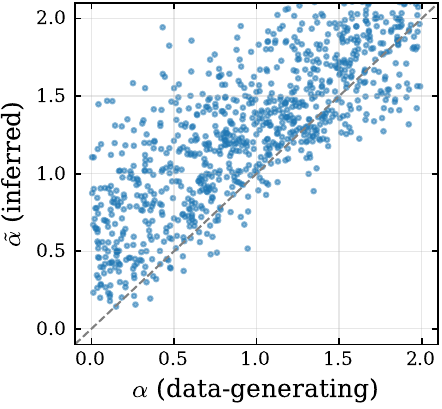}
            \\
            \hspace{1em}\texttt{fsam} \footnotesize{($\text{RMSE}=0.40$)}
        \end{minipage}
        \\[2ex]
        (b) \textsc{\textsc{duffing oscillator}} \small{($\theta = \{\alpha\}$)}
    \end{minipage}
    \\[4ex]
    \begin{minipage}[t]{0.99\textwidth}
        \centering
        \vspace*{0pt}
        \begin{minipage}[t]{0.49\linewidth}
            \centering
            \vspace*{0pt}
            \includegraphics[clip,width=\linewidth]{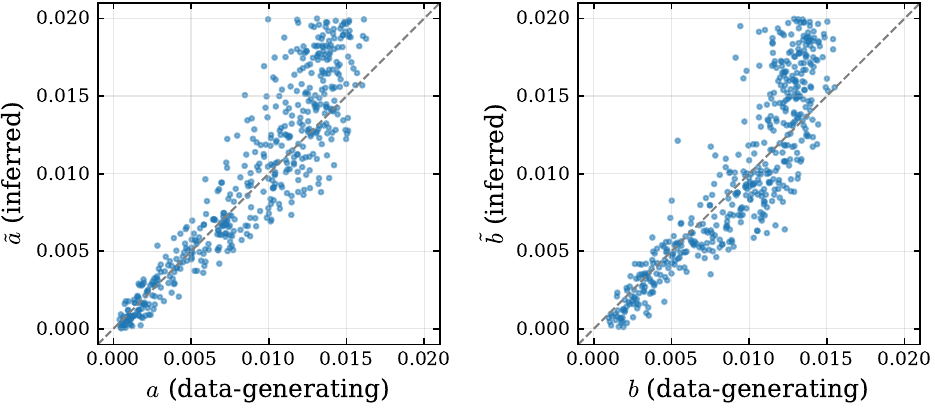}
            \\
            \hspace{1em}\texttt{erm} \footnotesize{($\text{RMSE}=0.0033$)}
        \end{minipage}
        \hfill
        \begin{minipage}[t]{0.49\linewidth}
            \centering
            \vspace*{0pt}
            \includegraphics[clip,width=\linewidth]{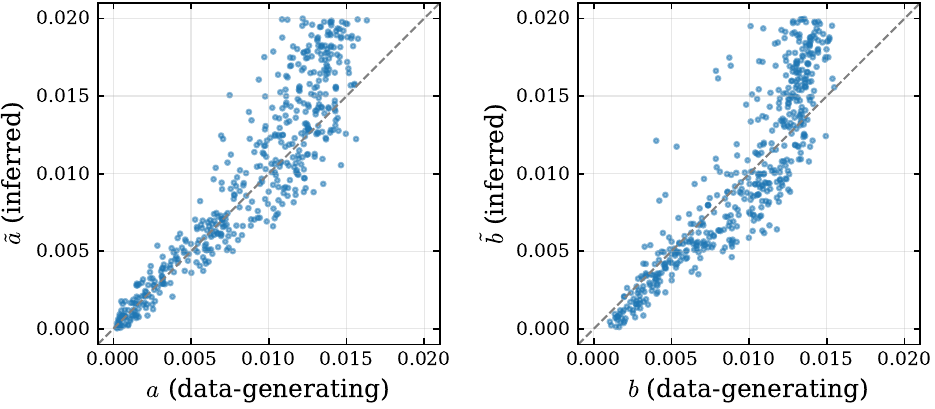}
            \\
            \hspace{1em}\texttt{fsam} \footnotesize{($\text{RMSE}=0.0033$)}
        \end{minipage}
        \\[2ex]
        (c) \textsc{\textsc{reaction-diffusion}} \small{($\theta = \{a,b\}$)}
    \end{minipage}
    \caption{Inferred values of the parameters of the scientific model part, $\theta$.}
    \label{fig:varphys}
\end{figure}

\subsection{Sensitivity to Hyperparameters}

In \cref{fig:sensitivity_full}, we report the full result of the sensitivity of parameters with regard to the hyperparameter configurations.

\clearpage

\begin{figure}[p]
    \centering
    \begin{minipage}[t]{\textwidth}
        \vspace*{0pt}
        \centering
        \hspace{0mm}
        \texttt{p-reg} \hspace{18mm}
        \texttt{f-reg} \hspace{20mm}
        \texttt{sam} \hspace{20mm}
        \texttt{asam} \hspace{19mm}
        \texttt{fsam}
    \end{minipage}
    \begin{minipage}[t]{\textwidth}
        \vspace*{0pt}
        \centering
        \includegraphics[trim={0 10pt 0 0},clip,width=0.19\linewidth]{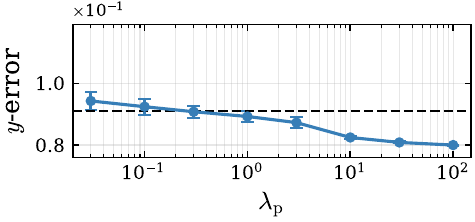}
        \hfill
        \includegraphics[trim={10pt 10pt 0 0},clip,width=0.19\linewidth]{figs/sensitivity/pendraw/reg-yerr.pdf}
        \hfill
        \includegraphics[trim={10pt 10pt 0 0},clip,width=0.19\linewidth]{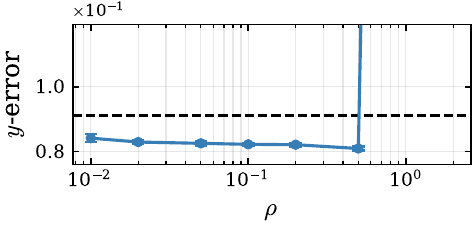}
        \hfill
        \includegraphics[trim={10pt 10pt 0 0},clip,width=0.19\linewidth]{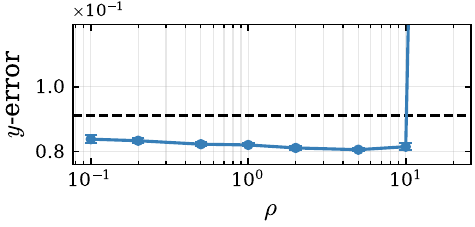}
        \hfill
        \includegraphics[trim={10pt 10pt 0 0},clip,width=0.19\linewidth]{figs/sensitivity/pendraw/fsam-yerr.pdf}
        \\
        \includegraphics[trim={0 0 0 0},clip,width=0.19\linewidth]{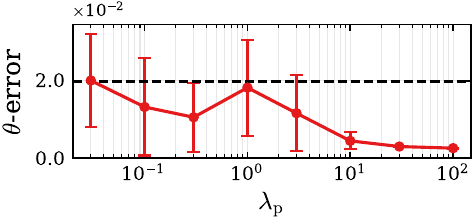}
        \hfill
        \includegraphics[trim={10pt 0 0 0},clip,width=0.19\linewidth]{figs/sensitivity/pendraw/reg-therr.pdf}
        \hfill
        \includegraphics[trim={10pt 0 0 0},clip,width=0.19\linewidth]{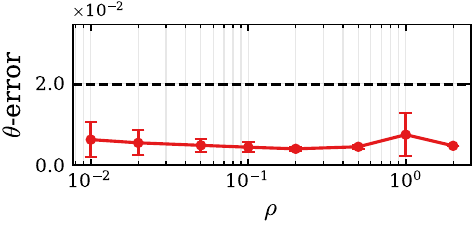}
        \hfill
        \includegraphics[trim={10pt 0 0 0},clip,width=0.19\linewidth]{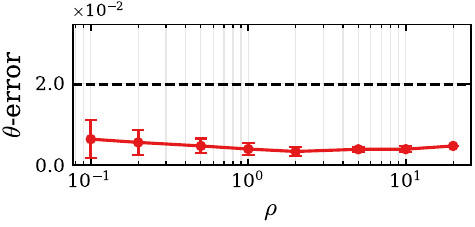}
        \hfill
        \includegraphics[trim={10pt 0 0 0},clip,width=0.19\linewidth]{figs/sensitivity/pendraw/fsam-therr.pdf}
        \\[-1ex]
        (a) \textsc{pendulum time-series}
    \end{minipage}
    \\[2ex]
    \begin{minipage}[t]{\textwidth}
        \vspace*{0pt}
        \centering
        \includegraphics[trim={0 10pt 0 0},clip,width=0.19\linewidth]{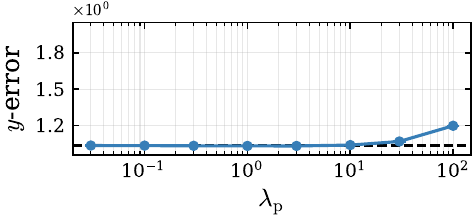}
        \hfill
        \includegraphics[trim={10pt 10pt 0 0},clip,width=0.19\linewidth]{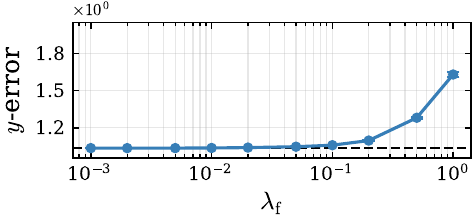}
        \hfill
        \includegraphics[trim={10pt 10pt 0 0},clip,width=0.19\linewidth]{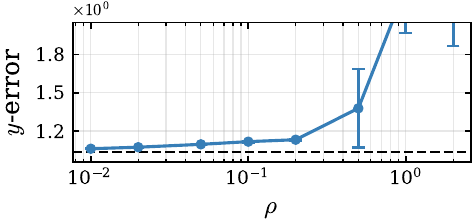}
        \hfill
        \includegraphics[trim={10pt 10pt 0 0},clip,width=0.19\linewidth]{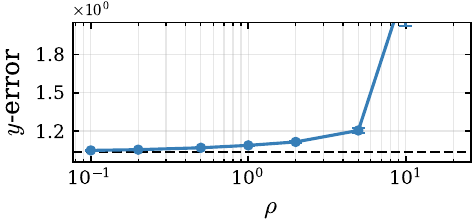}
        \hfill
        \includegraphics[trim={10pt 10pt 0 0},clip,width=0.19\linewidth]{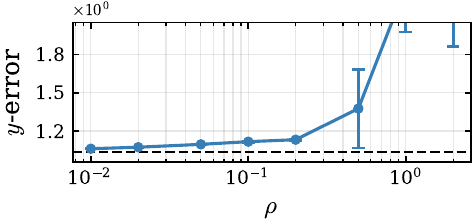}
        \\
        \includegraphics[trim={0 0 0 0},clip,width=0.19\linewidth]{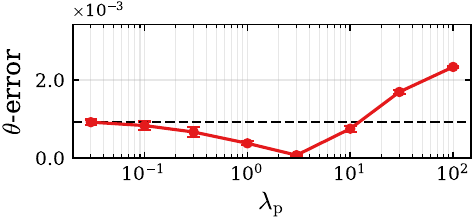}
        \hfill
        \includegraphics[trim={10pt 0 0 0},clip,width=0.19\linewidth]{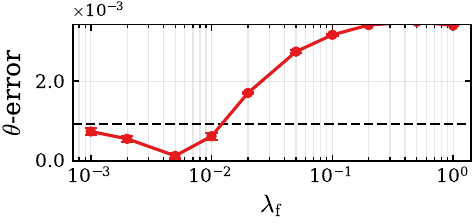}
        \hfill
        \includegraphics[trim={10pt 0 0 0},clip,width=0.19\linewidth]{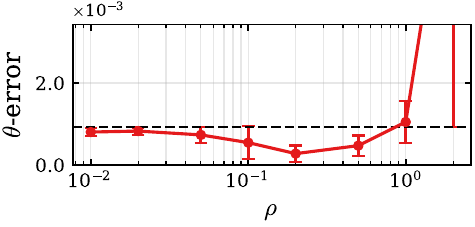}
        \hfill
        \includegraphics[trim={10pt 0 0 0},clip,width=0.19\linewidth]{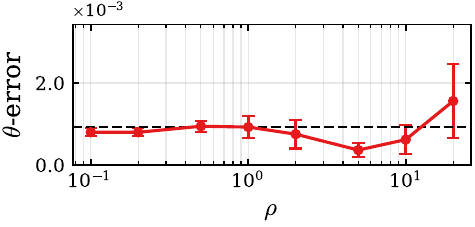}
        \hfill
        \includegraphics[trim={10pt 0 0 0},clip,width=0.19\linewidth]{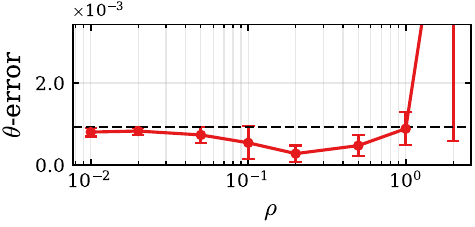}
        \\[-1ex]
        (b) \textsc{reaction-diffusion}
    \end{minipage}
    \\[2ex]
    \begin{minipage}[t]{\textwidth}
        \vspace*{0pt}
        \centering
        \includegraphics[trim={0 10pt 0 0},clip,width=0.19\linewidth]{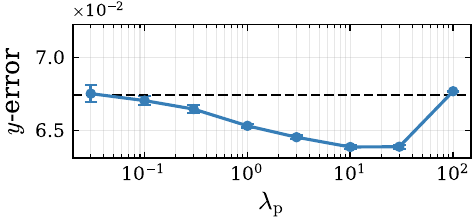}
        \hfill
        \includegraphics[trim={10pt 10pt 0 0},clip,width=0.19\linewidth]{figs/sensitivity/duffing/reg-yerr.pdf}
        \hfill
        \includegraphics[trim={10pt 10pt 0 0},clip,width=0.19\linewidth]{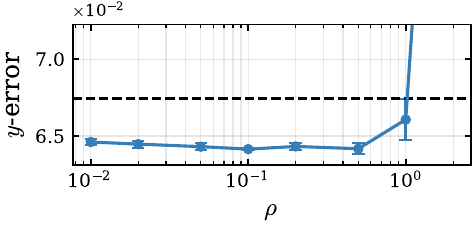}
        \hfill
        \includegraphics[trim={10pt 10pt 0 0},clip,width=0.19\linewidth]{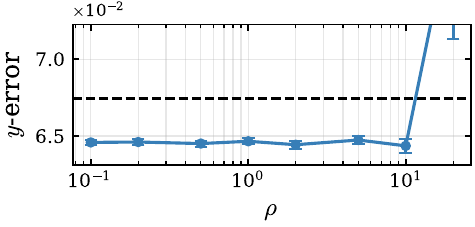}
        \hfill
        \includegraphics[trim={10pt 10pt 0 0},clip,width=0.19\linewidth]{figs/sensitivity/duffing/fsam-yerr.pdf}
        \\
        \includegraphics[trim={0 0 0 0},clip,width=0.19\linewidth]{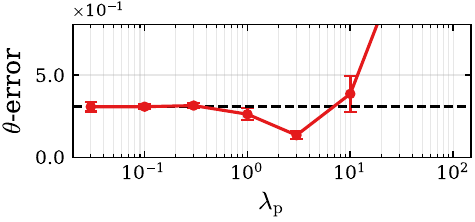}
        \hfill
        \includegraphics[trim={10pt 0 0 0},clip,width=0.19\linewidth]{figs/sensitivity/duffing/reg-therr.pdf}
        \hfill
        \includegraphics[trim={10pt 0 0 0},clip,width=0.19\linewidth]{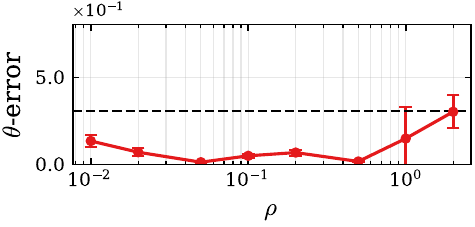}
        \hfill
        \includegraphics[trim={10pt 0 0 0},clip,width=0.19\linewidth]{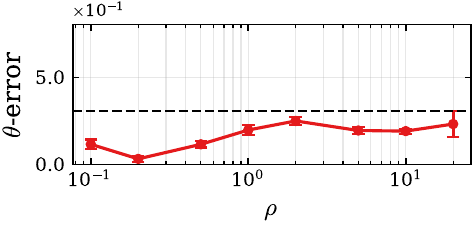}
        \hfill
        \includegraphics[trim={10pt 0 0 0},clip,width=0.19\linewidth]{figs/sensitivity/duffing/fsam-therr.pdf}
        \\[-1ex]
        (c) \textsc{Duffing oscillator}
    \end{minipage}
    \\[2ex]
    \begin{minipage}[t]{\textwidth}
        \vspace*{0pt}
        \centering
        \includegraphics[trim={0 10pt 0 0},clip,width=0.19\linewidth]{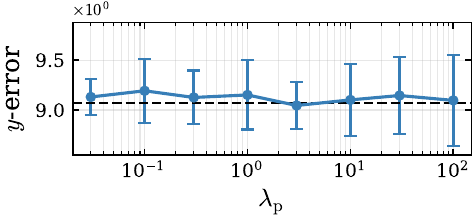}
        \hfill
        \includegraphics[trim={10pt 10pt 0 0},clip,width=0.19\linewidth]{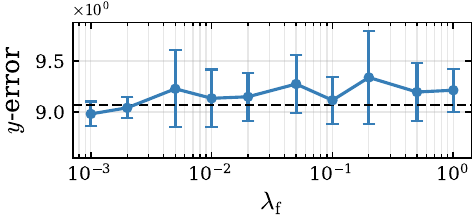}
        \hfill
        \includegraphics[trim={10pt 10pt 0 0},clip,width=0.19\linewidth]{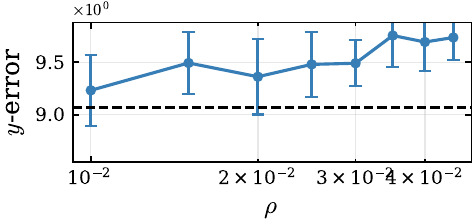}
        \hfill
        \includegraphics[trim={10pt 10pt 0 0},clip,width=0.19\linewidth]{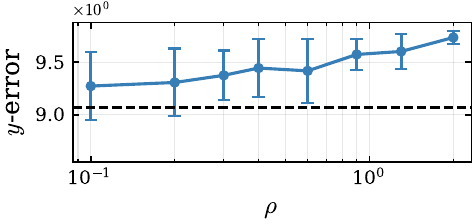}
        \hfill
        \includegraphics[trim={10pt 10pt 0 0},clip,width=0.19\linewidth]{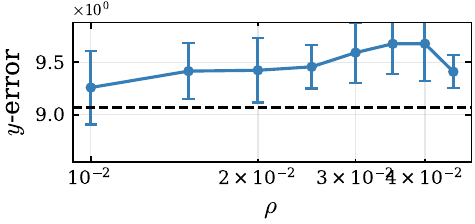}
        \\
        \includegraphics[trim={0 0 0 0},clip,width=0.19\linewidth]{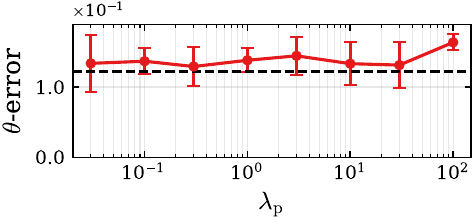}
        \hfill
        \includegraphics[trim={10pt 0 0 0},clip,width=0.19\linewidth]{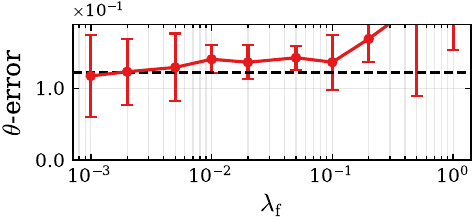}
        \hfill
        \includegraphics[trim={10pt 0 0 0},clip,width=0.19\linewidth]{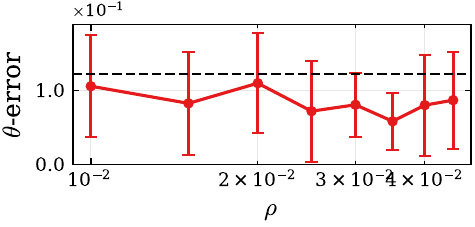}
        \hfill
        \includegraphics[trim={10pt 0 0 0},clip,width=0.19\linewidth]{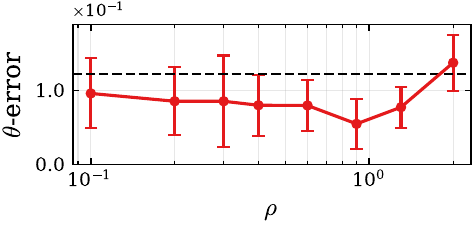}
        \hfill
        \includegraphics[trim={10pt 0 0 0},clip,width=0.19\linewidth]{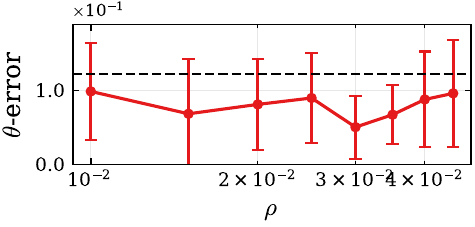}
        \\[-1ex]
        (d) \textsc{pendulum images}
    \end{minipage}
    \\[2ex]
    \begin{minipage}[t]{\textwidth}
        \vspace*{0pt}
        \centering
        \includegraphics[trim={0 10pt 0 0},clip,width=0.19\linewidth]{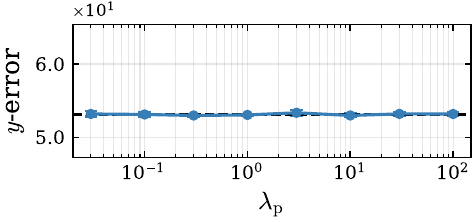}
        \hfill
        \includegraphics[trim={10pt 10pt 0 0},clip,width=0.19\linewidth]{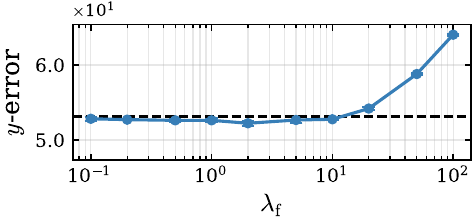}
        \hfill
        \includegraphics[trim={10pt 10pt 0 0},clip,width=0.19\linewidth]{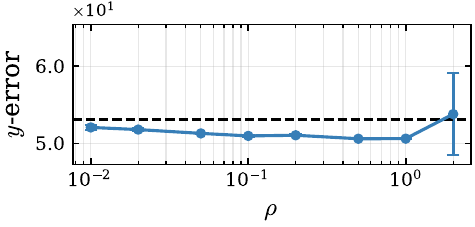}
        \hfill
        \includegraphics[trim={10pt 10pt 0 0},clip,width=0.19\linewidth]{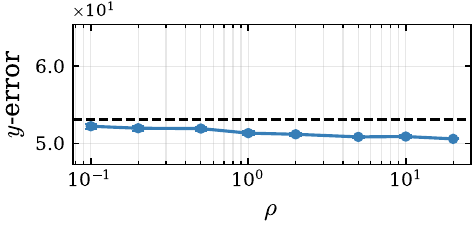}
        \hfill
        \includegraphics[trim={10pt 10pt 0 0},clip,width=0.19\linewidth]{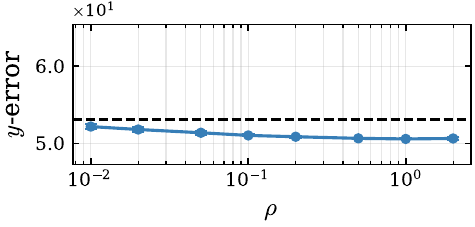}
        \\
        \includegraphics[trim={0 0 0 0},clip,width=0.19\linewidth]{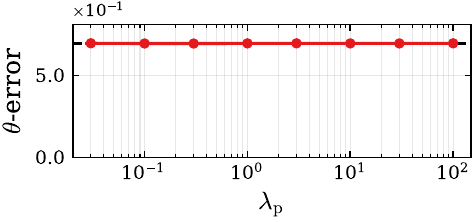}
        \hfill
        \includegraphics[trim={10pt 0 0 0},clip,width=0.19\linewidth]{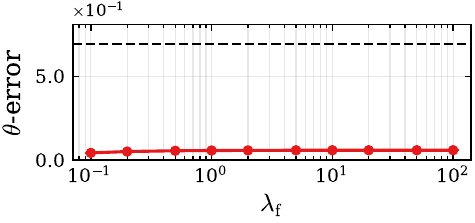}
        \hfill
        \includegraphics[trim={10pt 0 0 0},clip,width=0.19\linewidth]{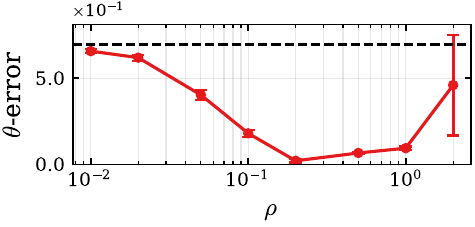}
        \hfill
        \includegraphics[trim={10pt 0 0 0},clip,width=0.19\linewidth]{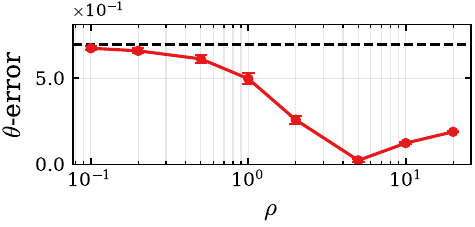}
        \hfill
        \includegraphics[trim={10pt 0 0 0},clip,width=0.19\linewidth]{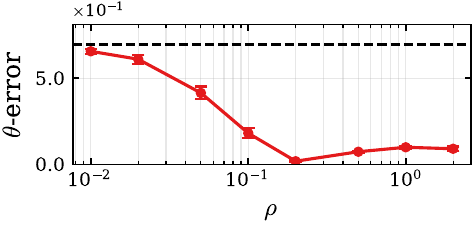}
        \\[-1ex]
        (e) \textsc{wind tunnel}
    \end{minipage}
    \\[2ex]
    \begin{minipage}[t]{\textwidth}
        \vspace*{0pt}
        \centering
        \includegraphics[trim={0 10pt 0 0},clip,width=0.19\linewidth]{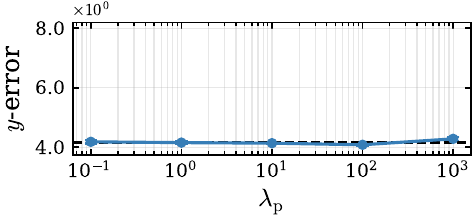}
        \hfill
        \includegraphics[trim={10pt 10pt 0 0},clip,width=0.19\linewidth]{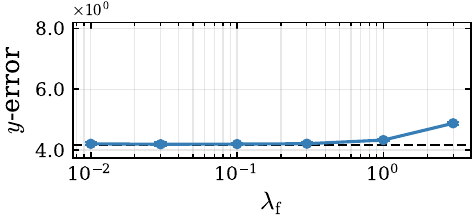}
        \hfill
        \includegraphics[trim={10pt 10pt 0 0},clip,width=0.19\linewidth]{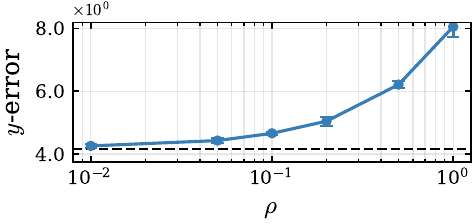}
        \hfill
        \includegraphics[trim={10pt 10pt 0 0},clip,width=0.19\linewidth]{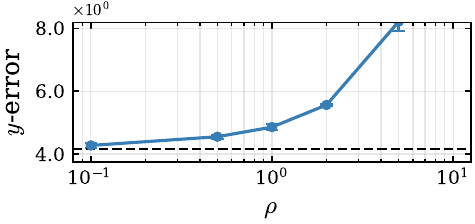}
        \hfill
        \includegraphics[trim={10pt 10pt 0 0},clip,width=0.19\linewidth]{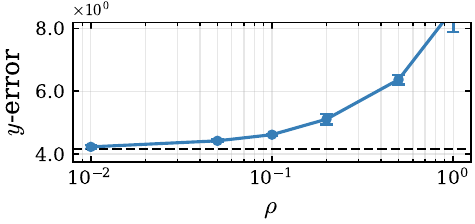}
        \\
        \includegraphics[trim={0 0 0 0},clip,width=0.19\linewidth]{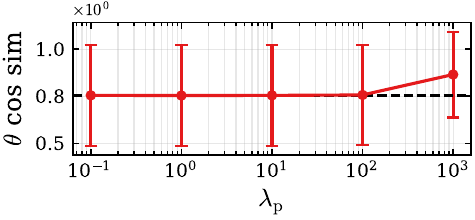}
        \hfill
        \includegraphics[trim={10pt 0 0 0},clip,width=0.19\linewidth]{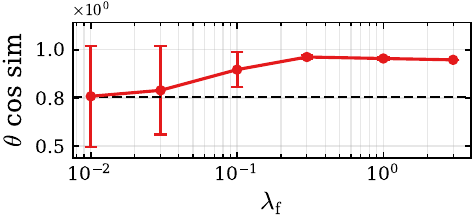}
        \hfill
        \includegraphics[trim={10pt 0 0 0},clip,width=0.19\linewidth]{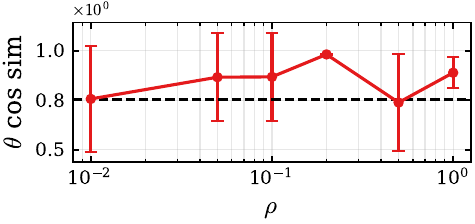}
        \hfill
        \includegraphics[trim={10pt 0 0 0},clip,width=0.19\linewidth]{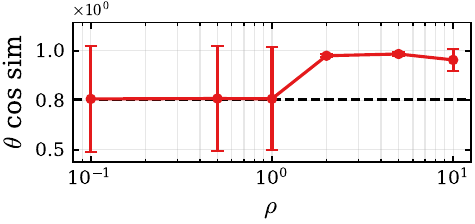}
        \hfill
        \includegraphics[trim={10pt 0 0 0},clip,width=0.19\linewidth]{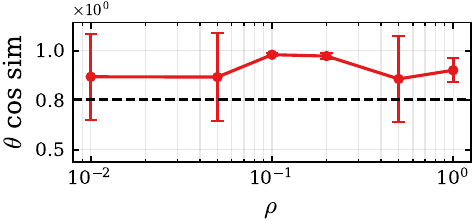}
        \\[-1ex]
        (f) \textsc{light tunnel} (cosine similarities for $\theta$)
    \end{minipage}
    \caption{Sensitivity of performance with regard to the hyperparameters.}
    \label{fig:sensitivity_full}
\end{figure}

%%%%%%%%%%%%%%%%%%%%%%%%%%%%%%%%%%%%%%%%%%%%%%%%%%%%%%%%%%%%

% \clearpage
% \input{checklist.tex}

\end{document}